\documentclass{article}

\usepackage[final]{neurips_2025}

\usepackage[utf8]{inputenc} 
\usepackage[T1]{fontenc}    
\usepackage{hyperref}       
\usepackage{url}            
\usepackage{booktabs}       
\usepackage{amsfonts}       
\usepackage{nicefrac}       
\usepackage{microtype}      
\usepackage{xcolor}         
\usepackage{amsfonts,amsmath,amsthm,amssymb}
\usepackage{thm-restate}
\usepackage{cleveref}
\usepackage{svg}
\usepackage{enumitem}

\title{The Logical Expressiveness of Temporal GNNs 
\\
via Two-Dimensional Product Logics
}

\author{%
  Marco S\"alzer \\
  RPTU Kaiserslautern-Landau\\
  Kaiserslautern, Germany\\
  \texttt{marco.saelzer@rptu.de}\\
  \And 
  Przemys{\l}aw Andrzej Wa{\l}\k{e}ga \\
  Queen Mary University of London, UK \\ 
University of Łódź, Poland \\
  \texttt{p.walega@qmul.ac.uk}
  \And 
  Martin Lange \\
  Theoretical Computer Science / Formal Methods \\
  University of Kassel, Germany \\
  \texttt{martin.lange@uni-kassel.de}
}

\usepackage{tikz}
\usetikzlibrary{decorations.pathreplacing}

\usepackage{pgfplots}
\pgfplotsset{compat=1.18}
\usetikzlibrary{arrows,backgrounds,automata,topaths,patterns,decorations.pathreplacing,decorations.pathmorphing}
\usetikzlibrary{arrows, chains, positioning, quotes, shapes.geometric}
\usetikzlibrary{calc}
\usetikzlibrary{fit}
\usetikzlibrary{arrows}
\usetikzlibrary{positioning,automata}

\newtheorem{theorem}{Theorem}
\newtheorem{lemma}{Lemma}
\newtheorem{corollary}[theorem]{Corollary}
\newtheorem{definition}[theorem]{Definition}

\definecolor{mygreen}{HTML}{9BF3AA}
\definecolor{myred}{HTML}{F39B9B}
\definecolor{myblue}{HTML}{9BB0F3}
\definecolor{mywhite}{HTML}{FFFFFF}

  \definecolor{SoftBlue}{HTML}{89CFF0}
  \definecolor{MintGreen}{HTML}{98FF98}
  \definecolor{LightCoral}{HTML}{F08080}
  \definecolor{SoftPurple}{HTML}{D8BFD8}
  \definecolor{PaleYellow}{HTML}{FFFFE0}
  \definecolor{SkyBlue}{HTML}{87CEEB}
  \definecolor{Peach}{HTML}{FFDAB9}
  \definecolor{Lavender}{HTML}{E6E6FA}
  \definecolor{LightPink}{HTML}{FFB6C1}
  \definecolor{Cream}{HTML}{FFFDD0}
  \definecolor{Aqua}{HTML}{00FFFF}
  \definecolor{Chartreuse}{HTML}{7FFF00}
  \definecolor{Crimson}{HTML}{DC143C}
  \definecolor{DarkOrange}{HTML}{FF8C00}
  \definecolor{Goldenrod}{HTML}{DAA520}
  \definecolor{HotPink}{HTML}{FF69B4}
  \definecolor{IndianRed}{HTML}{CD5C5C}
  \definecolor{LightSeaGreen}{HTML}{20B2AA}
  \definecolor{MediumPurple}{HTML}{9370DB}
  \definecolor{Salmon}{HTML}{FA8072}

\definecolor{LightGray}{RGB}{150,150,150}
\definecolor{Green1}{RGB}{171,255,177}
\definecolor{Green2}{RGB}{65,243,77}
\definecolor{Green3}{RGB}{120,165,101}
\definecolor{Brown}{RGB}{243,190,130}
\definecolor{Blue1}{RGB}{93,114,164}
\definecolor{Blue2}{RGB}{213,223,245}
\definecolor{Blue3}{RGB}{42,69,176}
\definecolor{Blue4}{RGB}{240,246,255}
\definecolor{Blue5}{RGB}{50,100,200}
\definecolor{Red1}{RGB}{255,89,100}
\definecolor{Light}{RGB}{102,205,170}

\usepackage{todonotes}

\newcommand{\N}{\ensuremath{\mathbb{N}}}

\newcommand{\Q}{\ensuremath{\mathbb{Q}}}
\newcommand{\R}{\ensuremath{\mathbb{R}}}

\newcommand*{\ldblbrace}{\{\mskip-5mu\{}
\newcommand*{\rdblbrace}{\}\mskip-5mu\}}

\newcommand{\con}{\ensuremath{\,||\,}}

\newcommand{\G}{\ensuremath{G}}

\newcommand{\x}{\ensuremath{\mathbf{x}}}

\newcommand{\h}{\ensuremath{\mathbf{h}}}

\newcommand{\com}{\ensuremath{\mathsf{comb}}}
\newcommand{\agg}{\ensuremath{\mathsf{agg}}}
\newcommand{\msg}{\mathsf{msg}}

\newcommand{\out}{\ensuremath{\mathsf{out}}}

\newcommand{\trelu}{\ensuremath{\mathsf{trReLU}}}

\newcommand{\FNN}[1]{\mathsf{FNN}[#1]}

\newcommand{\TG}{\ensuremath{TG}}

\newcommand{\model}{\ensuremath{T}}

\newcommand{\idTGNN}{\mathcal{T}_{\textsf{rec}}}
\newcommand{\tandgTGNN}{\mathcal{T}_{\textsf{TandG}}}
\newcommand{\cell}{\mathit{Cell}}
\newcommand{\globTGNN}{\mathcal{T}_{\textsf{glob}}}

\newcommand{\sub}{\mathit{sub}} 

\newcommand{\prev}[1]{\mathsf{Y}_{#1}}

\newcommand{\past}{\mathsf{P}}

\newcommand{\PTL}{\ensuremath{\mathsf{PTL}}}

\newcommand{\PTLp}{\ensuremath{\mathsf{PTL}_{\past{},\prev{}}}}
\newcommand{\PTLK}{\ensuremath{\mathsf{PTL}_{\past{},\prev{}}{\times}\mathsf{K}}}

\newcommand{\K}{\ensuremath{\mathsf{K} }}

\newcommand{\logbenedikt}{\mathcal{L}\text{-}\mathsf{MP}^2}
\newcommand{\tlogbenedikt}{\mathsf{PTL}_{\past{},\prev{}}{\times}(\logbenedikt)} 
\newcommand{\lognunn}{K^{\#}}
\newcommand{\tlognunn}{\mathsf{PTL}_{\past{},\prev{}}{\times}K^{\#}}

\newcommand{\tandg}{\mathsf{TandG}}

\begin{document}

\maketitle

\begin{abstract}
In recent years, the expressive power of various neural architectures---including graph neural networks (GNNs), transformers, and recurrent neural networks---has been characterised using tools from logic and formal language theory. As the capabilities of basic architectures are becoming well understood, increasing attention is turning to models that combine multiple architectural paradigms. Among them particularly important, and challenging to analyse, are temporal extensions of GNNs, which integrate both spatial (graph-structure) and temporal (evolution over time) dimensions. In this paper, we initiate the study of logical characterisation of temporal GNNs by connecting them to two-dimensional product logics. We show that the expressive power of temporal GNNs depends on how graph and temporal components are combined. In particular, temporal GNNs that apply static GNNs recursively over time can capture all properties definable in the product logic of (past) propositional temporal logic PTL and the modal logic K. In contrast, architectures such as graph-and-time TGNNs and global TGNNs can only express restricted fragments of this  logic, where the interaction between temporal and spatial operators is syntactically constrained. These provide us with the first results on the logical expressiveness of temporal GNNs.

\end{abstract}

\section{Introduction}
\label{sec:intro}

Recent years have seen significant progress in understanding the expressive power of neural architectures using tools 
from logic, formal language theory, and graph theory. 
Some of the most prominent results concern \emph{Graph Neural Networks} (GNNs),
whose distinguishing power has been famously characterised with the Weisfeiler-Leman isomorphism 
test \cite{morris2019weisfeiler,DBLP:conf/iclr/XuHLJ19} and whose logical expressiveness has been 
captured with modal and first-order logics \cite{BarceloKM0RS20, Grohe21, NunnSST24, BenediktLMT24, ahvonen2024logical}. 
These insights have revealed both the limitations and strengths of GNNs,  inspiring the development of more expressive variants 
such as  higher-order GNNs~\cite{morris2019weisfeiler}. 
They have also opened the way for extracting logical rules from GNNs, 
advancing explainability in graph-based learning \cite{DBLP:conf/kr/CucalaGMK23,DBLP:conf/kr/CucalaG24}. 
Consequently, the logical expressiveness of GNNs has become a rapidly evolving research area.

As the capabilities of standard GNN architectures become relatively well understood, research interest is shifting to more complex architectures, which combine multiple dimensions of structure. One particularly prominent and challenging case is that of 
\emph{Temporal Graph Neural Networks} (TGNNs) \cite{LongaLSBLLSP23,DBLP:journals/access/SkardingGM21,GR22}, which can be seen as an extension of 
GNNs enabling to process \emph{temporal graphs}, that is, graphs whose topology evolves over time. 
As a result, TGNN computations combine the spatial  (graph structure) with temporal (changes in time) dimensions.
Considerable advancements have been made in the design of various
TGNN architectures \cite{LongaLSBLLSP23} and their deployment across diverse applications
such as traffic forecasting, financial applications, and epidemiological contexts
\cite{YuYZ18, ParejaDCMSKKSL20, Kapoor20}.
However 
understanding  expressive capabilities of TGNNs remains limited, with the first steps only comparing some TGNN architectures \cite{GR22, chen2023calibrate} or establishing relations with temporal versions of the Weisfeiler-Leman test \cite{SouzaMKG22,WR24}.
To the best of our knowledge no connection of TGNNs with logics has been established prior to this work.

\paragraph{Our contribution.}
In this work, we initiate the analysis of the logical expressiveness of TGNNs.
Our goal is to characterise the temporal graph properties that TGNNs can express using logical languages.
To achieve this, we propose a novel approach---analysing the logical expressiveness of TGNNs using 
two-dimensional product logics. 
This is based on the key insight that the computations of TGNNs, which process both the 
structural properties of static graphs and their evolution over time, naturally correspond to combinations of modal and temporal properties.

Both modal and temporal logics have been extensively studied in the logic community for decades, and  more recent work has connected them to the expressiveness of neural architectures \cite{BarceloKM0RS20, NunnSST24,DBLP:conf/iclr/BarceloKLP24,DBLP:journals/corr/abs-2404-04393}.
Likewise, many-dimensional logics, especially products of modal and temporal logics, have a rich theoretical foundation and well-established tools for analysing definability and complexity \cite{kurucz2003many,DBLP:journals/apin/BennettCWZ02,DBLP:conf/time/LutzWZ08}.
We are the first  to apply product logics as a framework for analysing the expressive power of TGNNs.


At a technical level, we show that certain TGNN architectures are 
(in-)capable of expressing certain combinations of modal-temporal properties. Our results are 
as follows:
\begin{itemize}[leftmargin=1em]
\item The class $\idTGNN[\hat{\mathcal{M}}]$  of TGNNs recursively applying  standard static GNNs (from class $\hat{\mathcal{M}}$)  over time, 
can express all properties definable in the product logic  $\PTLK$, which combines the seminal (past) 
propositional temporal logic $\mathsf{PTL}_{\mathsf{P}, \mathsf{Y}}$ and the modal logic $\mathsf{K}$ (Theorem~\ref{sec:logic2tgnn;thm:ptlk}).

\item Analogous results hold if we replace  $\hat{\mathcal{M}}$ and  $\mathsf{K}$ with matching expressiveness GNNs and  logics.
In particular, we show that such results hold for two specific pairs of static GNNs and logics, recently studied in literature \cite{NunnSST24,BenediktLMT24} (Theorem~\ref{sec:log2tgnn;thm:benediktnunn}).



\item
In contrast to recursive TGNNs, the class of time-and-graph TGNNs \cite{GR22} does not allow us to express all properties expressible in $\PTLK$ 
(Theorem~\ref{th:TandGweak}).
We show, however, that time-and-graph TGNNs can express all properties definable in a fragment of $\PTLK$, in which the allowed interplay between temporal and modal operators is syntactically restricted (Theorem~\ref{sec:variants_logic2tnn;thm:tandg}).


\item The class of global TGNNs \cite{WR24} also does not allow us to express all properties expressible in $\PTLK$ (Theorem~\ref{thm:globTGNNweak}).
As in the previous case,
we determine a fragment of $\PTLK$, such that global TGNNs  can express all properties definable in this fragment (Theorem~\ref{sec:variants_logic2tnn;thm:globaltgnn}).
This fragment, however, allows for different interactions between temporal and modal operators, than the fragment from Theorem~\ref{th:TandGweak}.


\item We show how our results allow to determine relative expressive power of TGNN classes considered in the paper  (\Cref{cor:exprrelation} and 
\Cref{thm:idstrongest}).

\end{itemize}

Beyond their theoretical significance, these results open the door to novel avenues in explainable 
and trustworthy AI. Just as logical characterisations of static GNNs have enabled the extraction of 
symbolic rules from neural models \cite{DBLP:conf/kr/CucalaGMK23,DBLP:conf/kr/CucalaG24}, our analysis 
lays the foundation for extracting temporal rules and dynamic specifications from TGNNs. 


\section{Related work}
\label{sec:intro;sub:related}

\paragraph{Temporal GNNs.}
There exists a plethora of TGNN models \cite{LongaLSBLLSP23}, differing on representation and processing  temporal  graphs.
There exist various classifications of TGNNs \cite{LongaLSBLLSP23}, but one of the main distinction is between snapshot-based models, where a temporal graph is given as a sequence of its timestamped snapshots
\cite{HajiramezanaliH19, SankarWGZY20, MicheliT22, YouDL22, Cini_Marisca_Bianchi_Alippi_2023},
and event-based models, where a temporal graph is given as a sequence of events modifying the graph structure \cite{XuRKKA20, rossi2020temporal, Luo22a}.
In this paper we  focus on the snapshot-based TGNNs, but it is worth to observe that in many settings these two types of temporal graph representation can be translated into each other.
Initial work on the expressive power of TGNNs has been conducted exploiting their relation to 
variants of the
Weisfeiler-Leman isomorphism test \cite{SouzaMKG22,WR24}
and by directly comparing expressiveness of particular  TGNN architectures \cite{GR22,chen2023calibrate}.
However, to the best of our knowledge, no work has yet studied the expressive power of TGNNs from the perspective 
of logics.


\paragraph{Logical expressiveness of static GNNs.}
The seminal logical characterisation of (message-passing) GNNs established that the properties expressible both by GNNs and first-order logic coincide with those definable in graded modal logic \cite{BarceloKM0RS20}. 
Subsequent work moved beyond the assumption of first-order expressibility, aiming to identify logics that capture the full expressive power of GNNs. This line of research led to the introduction of logics such as the modal logic 
$\mathsf{K}^{\#}$
\cite{NunnSST24} and logics extended with Presburger quantifiers \cite{BenediktLMT24}. More recent results have provided logical characterisations of further GNN architectures using logics 
with quantised parameters \cite{ahvonen2024logical, SST25} or  standard modal logics~\cite{grau2025correspondence}.
There are also results on the logical expressiveness of GNNs extended with global readouts \cite{BarceloKM0RS20,hauke2025aggregate}.

\paragraph{Product logics} 
Multi-dimensional product logics \cite{kurucz2003many,DBLP:books/daglib/0093551} offer a general framework for combining multiple modal logics, each capturing distinct aspect of reasoning such as space, time, or knowledge. 
In these systems, states are represented as tuples drawn from the component logics, and accessibility relations are defined componentwise, enabling interaction between different modal dimensions.
There is a long-standing tradition of studying the complexity and axiomatisation of multi-dimensional logics—a line of research that often proves to be highly challenging \cite{kurucz2003many,DBLP:books/daglib/0093551}. 
Two-dimensional product logics, in particular, have been extensively investigated due to their applicability in spatio-temporal \cite{DBLP:journals/apin/BennettCWZ02}, temporal-epistemic \cite{DBLP:journals/jcss/HalpernV89},
temporal-standpoint \cite{DBLP:conf/kr/GiganteAL23,demri2024computational},
and temporal description logic \cite{DBLP:journals/amai/ArtaleF00} settings.
Of particular relevance to our work is the two-dimensional product logic $\mathsf{PTL} \times \mathsf{K}$, which combines propositional temporal logic (PTL) \cite{Pnueli:1977} with the modal logic $\mathsf{K}$ \cite{Kripke1972-KRINAN-2}. This logic, along with its various extensions and variants, has been the subject of extensive study in the literature~\cite{kurucz2003many}.









\section{Notation}
\label{sec:fund}

We will briefly describe  basic notions, notation, and conventions used in the paper.
%
%

We use bold symbols, such as $\boldsymbol{x}$ and $\boldsymbol{y}$,
to represent vectors; we  assume that they are given as column vectors.
In particular, we use $\boldsymbol{0}$ for vectors containing only 0s; their dimension will usually be clear from context.
We let $\boldsymbol{x}_i$ be the 
$i$th element of $\boldsymbol{x}$,
and 
$\boldsymbol{x} \con \boldsymbol{y}$ the concatenation of $\boldsymbol{x}$ and $\boldsymbol{y}$, that is,  a vector obtained by stacking 
$\boldsymbol{x}$ onto $\boldsymbol{y}$. 
We use $\ldblbrace \dotsb \rdblbrace$ to denote multisets, that is, sets with possibly multiple occurrences of elements.

A (static, undirected, node-labelled) \emph{graph}, $G = (V, E, c)$, is a triple consisting of  a finite set $V$  of nodes,
a set $E$  of undirected edges over $V$, 
and a labelling $c: V \to \R^k$ of 
nodes  with vectors in $\R^k$.
If $c: V \to \{0,1\}^k$, we call vector entries 
 \emph{colours} and say that the graph is \emph{coloured} (with $k$ colours).
A \emph{pointed graph} is a pair $(G,v)$ consisting of a graph and one of its nodes.

A \emph{temporal graph}, $\TG = (\G_1,t_1), \dots, (\G_n,t_n)$, is a finite sequence of pairs $(G_i, t_i)$,
where $\G_i=(V_i,E_i,c_i)$ is a static graph such that $V_1 = \dotsb = V_n$ and $t_i \in \R$ is a number, 
also called a \emph{timestamp}, such that $t_1 < \dotsb < t_n$. We call $n$ the \emph{length} of $\TG$.
We usually refer to the set of nodes of $\TG$ as $V$.
We call a temporal graph \emph{discrete}, if 
 $t_i = i$ for each timestamp $t_i$.
A \emph{timestamped node} is a pair $(v, t_i)$ of a node and a timestamp; 
a \emph{pointed temporal graph} is a pair 
$(\TG,(v,t_i))$ of a temporal graph and a timestamped node. 
Often, we will write $(\TG,v)$ instead of $(\TG,(v,t_n))$.

\section{Message-passing GNNs and TGNNs}
\label{sec:tgnn}

In this section, we briefly present the standard message-passing GNNs \cite{GilmerSRVD17}, referred to as message passing neural networks (MPNNs) here, but also 
called aggregation-combine GNNs (AC-GNNs)~\cite{BarceloKM0RS20}.
Afterwards, we present three ways of extending them to the temporal setting, which gives rise to the three classes of temporal GNNs we will study in this paper.


\begin{definition}
    \label{sec:tgnn;def:mplayer}
A \emph{message-passing GNN} (MPNNs), $M = (l^1, \dotsc, l^k)$, is a finite sequence of 
\emph{message-passing layers}  of the form $l^i = (\com^i, \agg^i)$, where $\com^i$
are \emph{combination} functions mapping pairs of vectors to single vectors, and $\agg^i$ are  \emph{aggregation} functions mapping multisets of vectors to single vectors. 
An application of $M$ to a graph $G=(V, E, c)$ yields embeddings  $\h^{(i)}_v$ computed for all $v \in V$ and $i +1 \leq k$ as follows:
$$
\h^{(0)}_v = c(v),
\qquad 
\h^{(i+1)}_v = \com^{i+1}(\h^{(i)}_v,  \agg^{i+1} ( \ldblbrace \h^{(i)}_u \mid \{v,u\} \in E \rdblbrace) ).
$$
For a pointed graph $(G,v)$, we let $M(G,v) = \h^{(k)}_v$ be the final embedding computed for $v$.
\end{definition}

Depending on the type of combination and aggregation functions we obtain various classes of MPNNs, usually written as 
$\mathcal{M}$.
In particular, we use $\hat{\mathcal{M}}$ for the class of MPNNs whose 
combination functions are realised by
feedforward neural networks with truncated-ReLU $\max(0,\min(1,x))$ as activation function and with aggregation  given by the entrywise sum $\agg(S) = \sum_{\boldsymbol{x}\in S} \boldsymbol{x}$.
We denote the class of these FNN by $\FNN{\trelu}$ (see Appendix~\ref{def:fnn} for a formal definition).
Another class of MPNNs we consider is $\hat{\mathcal{M}}_\msg$ which is similar to $\hat{\mathcal{M}}$, but 
aggregations 
are now given by $\agg(S) = \sum_{\boldsymbol{x} \in S} \msg(\boldsymbol{x})$, where $\msg$ is a feedforward neural network from
$\FNN{\trelu}$.


Next, we discuss three classes of TGNNs obtained by extending MPNNs in order to process temporal graphs $\TG = (\G_1,t_1), \dots, (\G_n,t_n)$. The first type, recursive TGNNs, starts processing $\TG$ by applying an MPNN to the first 
static graph $\G_1$.
It then extends $\G_2$ by concatenating node labels with the embeddings computed by the MPNN.
Next, it applies the same MPNN to the obtained static graph.
This process is applied recursively in $n$ rounds, until all graphs $\G_1, \dots, \G_n$ are processed.
The formal definition of such models is given below.


\begin{definition}
    \label{sec:tgnn;def:idealtgnn}
\emph{Recursive TGNNs} $\idTGNN[\mathcal{M}]$, for a class $\mathcal{M}$ of MPNNs,
are pairs $\model = (M, \out)$, where $M \in \mathcal{M}$ 
and $\out$ is an \emph{output} function mapping vectors to binary values in $\{0,1\}$.
An application of $\model = (M, \out)$ to a temporal graph 
$\TG = (\G_1,t_1), \dots, (\G_n,t_n)$ with $\G_j=(V_j,E_j,c_j)$,
yields the following  embeddings for all $v \in V$ and $j+1 \leq n$:
\begin{align*}
 \h_v^{(0)}(t_1) & = c_1(v) \con \boldsymbol{0}, \qquad
 \h_v^{(0)}(t_{j+1})  = c_{j+1}(v) \con \h_v^{(k)}(t_{j}),
\\
\h_v^{(k)}(t_j) & = M((V_j, E_j, [u \mapsto \h_u^{(0)}(t_j)]),v).
\end{align*}
where $k$ is the number of layers in $M$ and $[u \mapsto \h_u^{(0)}(t_j)]$ is a labelling function mapping each $u \in V_j$ to the vector $\h_u^{(0)}(t_j)$,
%
\end{definition}
We remark that we introduce the $\idTGNN$ architecture to represent approaches that directly 
leverage existing MPNN architectures for temporal graphs \cite{YouDL22} without employing a specialised 
architecture.

The second type, time-and-graph TGNNs,
are studied in several papers \cite{GR22,chen2023calibrate} and akin to several models \cite{LiHCSWZP19, ChenWX22, SeoDVB18}.
Such TGNNs perform computations by exploiting not only MPNNs, but also \emph{cell functions} (e.g. a gated recurrent unit \cite{ChungGCB14}) mapping vectors to vectors.
In particular, time-and-graph TGNNs use two MPNNs ($M_1$ and $M_2$) and one cell function ($\mathit{Cell}$).
A temporal graph $\TG = (\G_1,t_1), \dots, (\G_n,t_n)$ is, again, processed from left to right in $n$ steps.
In each step $j+1$, the TGNN applies $M_1$ to $G_{j+1}$ and $M_2$ to $G_{j+1}$ with node labels replaced by embeddings computed in step $j$.
This results in computing two vectors for each node, which are combined into a single vector using $Cell$. 
After $n$ steps of such processing, the TGNN terminates its computations. 

\begin{definition}
    \label{sec:tgnn;def:tandgtgnn}
 \emph{Time-and-graph} TGNNs
 $\tandgTGNN[\mathcal{M}, \mathcal{C}]$, for a class $\mathcal{M}$ of MPNNs and a class $\mathcal{C}$ of 
 cell functions, are tuples $\model = (M_1, M_2, \cell, \out)$, where $M_1, M_2 \in \mathcal{M}$ have 
 the same number of layers, $\mathit{Cell} \in \mathcal{C}$, and $\out$ is an output 
 function (as in \Cref{sec:tgnn;def:idealtgnn}).
An application of $\model = (M_1, M_2, \cell, \out)$ to a temporal graph 
$\TG = (\G_1,t_1), \dots, (\G_n,t_n)$ with $\G_j=(V_j,E_j,c_j)$,
yields the following  embeddings for all $v \in V$ and $j+1 \leq n$:
\begin{align*}
\h_v(t_1) & = 
\cell(M_1(G_1,v), M_2((V_1,E_1,[u \mapsto \boldsymbol{0}]), v)), 
 \\ 
\h_v(t_{j+1}) & =             \cell(M_1(G_{j+1},v), M_2((V_{j+1},E_{j+1},[u \mapsto \h_u(t_{j})]), v)), 
    \end{align*}
where $[u \mapsto \h_u(t_{j})]$ is as described in    \Cref{sec:tgnn;def:idealtgnn}.
%
\end{definition}


The third class, global TGNNs  \cite{WR24}, is also realised by several TGNN models \cite{rossi2020temporal, Luo22a}.
Instead of processing a temporal graph in each time point, global TGNNs exploit a temporal message-passing mechanism.
It allows messages to be passed among nodes from different time points.
To capture the time difference between nodes between which messages are passed, global TGNNs use \emph{time functions} $\phi: \R \rightarrow \R^m$.


\begin{definition}
\emph{Global} TGNNs $\globTGNN[\mathcal{M}, \mathcal{Q}, \circ]$, 
for a class $\mathcal{M}$ of MPNNs, a class $\mathcal{Q}$ of time functions, 
and an operation $\circ$ combining two vectors of the same dimensionality,
are tuples 
$\model = (M, \phi, \out)$ where $M \in \mathcal{M}$ 
$\phi \in \mathcal{Q}$, and $\out$ is an output function (as in \Cref{sec:tgnn;def:idealtgnn}).
An application of $\model = (M, \phi, \out)$ to a temporal graph 
$\TG = (\G_1,t_1), \dots, (\G_n,t_n)$ with $\G_j=(V_j,E_j,c_j)$,
yields   embeddings $\h^{(i+1)}_v(t_j)$ for all $v \in V$, $i+1 \leq k$ (for $k$ the number of layers in $M$), and $j \leq n$:
    \begin{align*}
        \h^{(0)}_v(t_j) & = c_j(v), \\
        \h^{(i+1)}_v(t_j) & = \com^{i+1}(\h^{(i)}_v(t_j), \agg^{i+1} \ldblbrace\h^{(i)}_u(t_h) \circ \phi(t_j-t_{h}) \mid \{v,u\} \in E_h, h \leq j \rdblbrace).
    \end{align*}

\end{definition}


In all TGNN models considered above, an output function $\out$ is used
to determine the final classification. 
Formally, for a TGNN  $\model$ and  a pointed temporal graph $(\TG, v)$, we let the output of $\model$ be 
$\model(\TG, v) = \out(\mathbf{h}_v^{(k)}(t_n))$.
We will assume that $\out$ functions are 
realised by a single-layer FNN from the class $\FNN{\trelu}$.




\section{Logical Expressiveness via Product Logics}

We will exploit two-dimensional-product logics as a tool for analysing the expressive power of TGNNs.
Specifically, we consider two-dimensional logics which are products of modal and temporal logics, allowing 
combinations of structural and temporal properties to be expressed.
Importantly, this approach provides an opportunity to control the interaction between logical operators corresponding to the two dimensions.
By modifying allowed interactions in product logics we will obtain logics that are suitable to analyse the expressive power of various classes of TGNNs.

Before introducing product logics, we specify the notion of the expressive power that we will consider in this paper.
It corresponds to the ``logical expressiveness'' also called ``uniform expressiveness'', which is broadly studied in the literature \cite{BarceloKM0RS20,hauke2025aggregate,NunnSST24,BenediktLMT24,grau2025correspondence} and is defined as follows.

\begin{definition}
\label{sec:logic2tgnn;def:leq}
Let $\mathcal{L}$ be a logical language, whose formulas $\varphi$ are evaluated at timestamped nodes $(v,t_n)$ of temporal graphs $\TG=(\G_1,t_1), \dots, (\G_n,t_n)$.
We  write $\TG,(v,t_n) \models \varphi$ if $\varphi$ holds at $(v,t_n)$  in $\TG$.
A TGNN class
$\mathcal{T}$ is \emph{at least as expressive as $\mathcal{L}$}, written $\mathcal{L} \leq \mathcal{T}$, 
if for every formula 
$\varphi$ of $\mathcal{L}$, there exists a model $\model \in \mathcal{T}$ such that, 
for all (coloured) pointed temporal graphs $(\TG,(v,t_n))$:
$$
\TG,(v,t_n) \models \varphi \text{ if and only if } T(\TG,v) = 1.
$$
\end{definition}
Informally, $\mathcal{L} \leq \mathcal{T}$ means that TGNNs from the class $\mathcal{T}$ are powerful enough to express all properties that can be written as (arbitrarily long and complex) formulas of the logics $\mathcal{L}$. 

In order to study the expressive power of TGNNs, we will 
establish their relation to two-dimensional product   logics
\cite{kurucz2003many,DBLP:journals/apin/BennettCWZ02,DBLP:conf/time/LutzWZ08}.
The first dimension will correspond to time and will allow us  to capture evolution of a graph in time.
The second dimension will correspond to the spatial 
structure of graphs and will allow us to express properties of the static graphs at fixed time points.
The product logics we consider here are combinations of the most prominent temporal and modal logics, namely \emph{propositional temporal logic} \PTL{} \cite{Pnueli:1977} and the \emph{basic modal logic} $\K$ \cite{Kripke1972-KRINAN-2}.

\paragraph{Temporal logic \PTLp{}.} We consider the past-time version \PTLp{} of \PTL{} which contains only 
$\past$ (``\emph{sometime in the past}'') and $\prev{}$ (``\emph{yesterday}'') as native temporal operators.
\PTLp{} formulas are built using propositions $c_1, \dots , c_k$ which stand for colours that nodes of 
temporal graphs can take, together with an unrestricted use of Boolean connectives $\neg$ for ``not'', 
$\land$ for ``and'',  $\to$ for ``if ... then ...'', and $\leftrightarrow$ for `` if and only if''.
Formulas of the logic \PTLp{} can state, for example, that a node has currently colour $c_1$ and sometime 
in the past it had colour $c_2$ in two consecutive timepoints. This is written as 
$c_1 \land  \past (c_2 \land \prev{} c_2)$.

\paragraph{Modal logic \K{}.}
For the static part we use the basic modal logic $\K$. Its formulas are similar to those of \PTLp{}, but 
instead of the temporal operators $\past$ and $\prev{}$ it features a single modal operator $\Diamond$.
The intuitive reading of $\Diamond \varphi$ is ``\emph{$\varphi$ holds at one of neighbours (via graph edges) of the current node}''.
Logic $\K$ can express, for example, that a node has an outgoing path of three hops leading to a node of colour $c_1$ or a two-hop path to a node whose colour is $c_2$ and not $c_3$: 
$\Diamond \Diamond \Diamond c_1 \lor \Diamond \Diamond  (c_2 \land \neg c_3)$.

\paragraph{Product logic \PTLK{}.}
The product of logics \PTLp{} and \K{} allows  us to write formulas  with any combination of operators from 
\PTLp{} and \K{}.
For example, $\varphi = c_1 \land (\past c_2) \land \Diamond((\neg c_1 \land c_2) \land Y(c_1 \land \neg c_2))$
is satisfied at a vertex $v$ in time point $t_n$ if $v$ is coloured $c_1$ at $t_n$ and at some
past timestamp $t_i < t_n$, $v$ was coloured with $c_2$. Moreover, there is a neighbour of $v$ 
that is currently coloured with $c_2$ but not $c_1$ at $t_n$, whereas in the preceding timestamp 
$t_{n-1}$ it was coloured with $c_1$ but not with $c_2$. An example of a pointed temporal graph 
$(\TG,v)$ satisfying $\varphi$ in $t_n=t_4$ is given in Figure~\ref{fig:ptlk_examp}.
From the technical perspective, the semantics of \PTLK{} is defined over Cartesian products of 
models of \PTLp{} and \K{}, which justifies the name ``product logic''.
A formal definition of syntax and semantics of this product logic is given in Appendix~\ref{def:ptlk}.

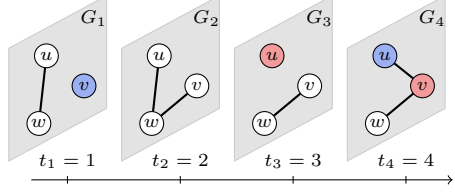
\begin{figure}[h]
    \centering
    \begin{tikzpicture}[
dot/.style = {draw, circle, minimum size=#1,
              inner sep=0pt, outer sep=0pt},
dot/.default = 6pt
]
\scriptsize
\pgfmathsetmacro{\tline}{-0.25}
\pgfmathsetmacro{\w}{1.3}
\pgfmathsetmacro{\h}{1.5}
\pgfmathsetmacro{\inh}{0.7}
\pgfmathsetmacro{\dist}{1.5}
\pgfmathsetmacro{\Ax}{0.5}
\pgfmathsetmacro{\Ay}{1.4}
\pgfmathsetmacro{\Bx}{0.4}
\pgfmathsetmacro{\By}{0.5}
\pgfmathsetmacro{\Cx}{1}
\pgfmathsetmacro{\Cy}{1}

\draw[->] (1.8,\tline) -- (7.4,\tline);
\node at (4.5,\tline-0.3) 
{ };

\foreach \x in {1,...,4}
{
\draw[fill=gray!90!black,opacity=0.2] (\x*\dist,0) -- (\x*\dist+\w,\inh) -- (\x*\dist+\w,\inh+\h) -- (\x*\dist,\h) -- cycle;

\node at (\x*\dist + 0.87 * \w, \inh+ \h - 0.3) {$G_{\x}$};
\node at (\x*\dist + 0.6 * \w, 0) {$t_{\x}=\x$};
\draw[-] (\x*\dist + 0.6 * \w, \tline+0.05) -- (\x*\dist + 0.6 * \w, \tline-0.05);

\node (A\x) at (\x*\dist+\Ax, 0+\Ay) {};
\node (B\x) at (\x*\dist+\Bx, 0+\By) {};
\node (C\x) at (\x*\dist+\Cx, 0+\Cy) {};
}

\draw[thick] (A1) -- (B1);
\draw[thick] (A2) -- (B2);
\draw[thick](B2) -- (C2);
\draw[thick](B3) -- (C3);
\draw[thick] (A4) -- (C4);
\draw[thick] (B4) -- (C4);

\node[dot=9pt,draw=black,fill=mywhite] at (A1) {$u$};
\node[dot=9pt,draw=black,fill=mywhite] at (B1) {$w$};
\node[dot=9pt,draw=black,fill=myblue] at (C1) {$v$};
\node[dot=9pt,draw=black,fill=mywhite] at (A2) {$u$};
\node[dot=9pt,draw=black,fill=mywhite] at (B2) {$w$};
\node[dot=9pt,draw=black,fill=mywhite] at (C2) {$v$};
\node[dot=9pt,draw=black,fill=myred] at (A3) {$u$};
\node[dot=9pt,draw=black,fill=mywhite] at (B3) {$w$};
\node[dot=9pt,draw=black,fill=mywhite] at (C3) {$v$};
\node[dot=9pt,draw=black,fill=myblue] at (A4) {$u$};
\node[dot=9pt,draw=black,fill=mywhite] at (B4) {$w$};
\node[dot=9pt,draw=black,fill=myred] at (C4) {$v$};
\end{tikzpicture}
    \caption{A temporal graph 
    of length $4$;
    colour $c_1$ is denoted by a red filling (node $u$ in $G_3$ and node $v$ in $G_4$) and 
        colour $c_2$ is denoted by a blue filling (node $v$ in $G_1$ and node $u$ in $G_4$)}
    \label{fig:ptlk_examp}
\end{figure}

We also consider variants of \PTLK{} obtained 
by replacing the modal logic $\K$ with other modal 
logics known from the literature, or by syntactically 
restricting how temporal and 
modal operators can co-occur in formulas. 
Such variants differentiate the expressive power of various  TGNN models.

\section{Logical expressiveness of recursive TGNNs}
\label{sec:logic2tgnn}

We provide expressive power results for recursive TGNNs by relating them with product logics.
We show that TGNNs from the class $\idTGNN[\hat{\mathcal{M}}]$ can express all properties definable in 
$\PTLK$.




\begin{restatable}{theorem}{ptlk}
    \label{sec:logic2tgnn;thm:ptlk}
    $\PTLK \leq \idTGNN[\hat{\mathcal{M}}]$.
\end{restatable}
\begin{proof}[Proof sketch.]
For each formula $\varphi \in \PTLK$, we enumerate its subformulas
$\varphi_1, \dots, \varphi_m,$ $\varphi_{m+1}, \dots, \varphi_n = \varphi$
in such a way that the $m$ atomic ones come first, and if $\varphi_i$ is a subformula of $\varphi_j$, written as $\varphi_i \in \mathit{sub}(\varphi_j)$, then $i \leq j$.
We construct a TGNN $\model_\varphi = (M, \out) \in \idTGNN[\hat{\mathcal M}]$ where $M$ consists of $n-m+1$ layers: one
layer for each non-atomic subformula and a final shift layer.
At each time $t$ and for each node $u$ of the current temporal graph, $\model_\varphi$ computes 
the hidden state $\h^{(n-m)}_u(t)$, which encodes
(\textsc{i}) the current truth values of all subformulas $\varphi_i$ in the first $n$ dimensions,
(\textsc{ii}) the truth values of all subformulas $\varphi_i$ at $t-1$ in dimensions $n+1$ to $2n$, and
(\textsc{iii}) the disjunction of truth values of subformulas $\varphi_i$ over all earlier time points
in dimensions $2n+1$ to $3n$.
The final shift layer (layer $n-m+1$) updates (\textsc{iii}) using (\textsc{ii}) without altering (\textsc{i}).
Correctness follows by nested induction over time and subformula structure, showing that 
for each $i \leq n$, dimension $i$ correctly tracks satisfaction of $\varphi_i$,
dimension $n+i$ tracks $\prev{} \varphi_i$, and dimension $2n+i$ tracks $\past{} \varphi_i$.
The output $\mathit{out}$ simply reads the value for $\varphi$ at the final time point.
A full proof is given in Appendix~\ref{proof:ptlk}.
\end{proof}


We provide two further analogous results obtained by replacing $\K$ and $\hat{\mathcal{M}}$ with pairs of 
logics and MPNNs of matching expressiveness known from the literature.
This suggests a general  connection between product logics and recursive TGNNs.
First, $\K$ can be replaced with the logic $\lognunn$, incorporating  linear arithmetics, and
$\hat{\mathcal{M}}$ with the class of MPNNs $\mathcal{M}_{\lognunn}$ of matching expressive power \cite{NunnSST24}.
Second, an analogous result is obtained when using the logic $\logbenedikt$, 
incorporating Presburger Arithmetic, and the class 
$\mathcal{O}\mathcal{L}\trelu\text{-GNN}$ of MPNNs, for which matching expressiveness 
results have been shown \cite{BenediktLMT24}. For the formal details of these logics, see Appendix~\ref{def:benediktnunn}, and we remark that it is evident that both strictly subsume $\PTLK$.

\begin{restatable}{theorem}{benediktnunn}
\label{sec:log2tgnn;thm:benediktnunn}
$\tlognunn \leq \idTGNN[\mathcal{M}_{\lognunn}]$
and
$\tlogbenedikt \leq \idTGNN[\mathcal{O}\mathcal{L}\trelu\text{-GNN}]$.
\end{restatable}
\begin{proof}[Proof sketch]
    We apply similar inductive argument as used in Theorem~\ref{sec:logic2tgnn;thm:ptlk}.
    For each formula $\varphi$, we enumerate its subformulas $\varphi_1, \dotsc, \varphi_{m_1},
    \varphi_{m_1+1}, \dotsc, \varphi_{m_k}, \varphi_{m_k+1}, \dotsc, \varphi_n = \varphi$,
    so that all subformulas of the form $\prev{}\psi$ or $\past{}\psi$ occupy positions $m_i$ for
    $i=1,\ldots,k$.
    Each group of purely non-temporal subformulas (situated between $m_i$ and $m_{i+1}$, assuming that 
    formulas $\varphi_i$
    with $i \leq m_i$ are already addressed) is captured inductively by stacks
    of MPNN layers whose existence is guaranteed by  results from literature \cite{BenediktLMT24,NunnSST24}.
    This necessitates that $\mathcal{O}\mathcal{L}\trelu$-GNN and $\mathcal{M}_{\lognunn}$ are closed under
    arbitrary but well-defined combinations of MPNN layers.
    Temporal subformulas $\prev{}\psi$ and $\past{}\psi$ are handled by simple dimension shifts,
    reading values stored at previous timepoints or across past histories, similar to Theorem~\ref{sec:logic2tgnn;thm:ptlk}.
    Here, we use the same form of hidden states, divided into three blocks tracking
    the current, previous, and past satisfaction of subformulas.
    Correctness follows by nested induction over time and subformula structure:
    for non-temporal subformulas, correctness can be shown exploiting results from literature \cite{BenediktLMT24,NunnSST24},
    while for temporal subformulas, it follows from the way in which the temporal operators
    $\prev{}$ and $\past{}$ are addressed.
    The full proof is  in Appendix~\ref{proof:benediktnunn}.
\end{proof}



\section{Logical expressiveness of time-and-graph and global TGNNs}
\label{sec:logic2tgnnvariants}

Next, we focus on two variants of 
TGNNs that have been considered in the literature thus far, aiming to unveil certain differences 
in their logical expressiveness. 

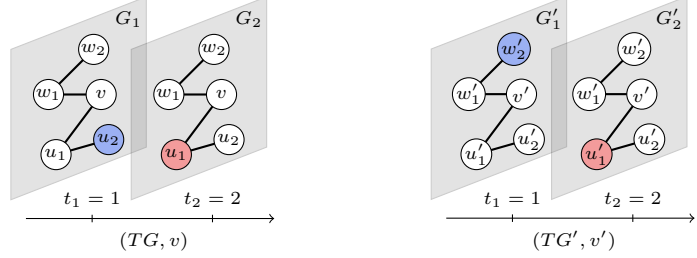
\begin{figure}
    \centering
    \begin{minipage}{0.4\textwidth}
    \centering
            \begin{tikzpicture}[
dot/.style = {draw, circle, minimum size=0.4cm,
              inner sep=0pt, outer sep=0pt},
dot/.default = 6pt
]
\scriptsize
\pgfmathsetmacro{\tline}{-0.25}
\pgfmathsetmacro{\w}{1.8}
\pgfmathsetmacro{\h}{2.0}
\pgfmathsetmacro{\inh}{0.7}
\pgfmathsetmacro{\dist}{1.6}
\pgfmathsetmacro{\Ax}{0.5}
\pgfmathsetmacro{\Ay}{1.4}
\pgfmathsetmacro{\Bx}{0.6}
\pgfmathsetmacro{\By}{0.6}
\pgfmathsetmacro{\Cx}{1.2}
\pgfmathsetmacro{\Cy}{1.4}
\pgfmathsetmacro{\Dx}{1.1}
\pgfmathsetmacro{\Dy}{2}
\pgfmathsetmacro{\Ex}{1.3}
\pgfmathsetmacro{\Ey}{0.8}


\draw[->] (1.8,\tline) -- (5.1,\tline);
\node at (3.5,\tline-0.3) {$(\TG,v)$};

\foreach \x in {1,...,2}
{
\draw[fill=gray!90!black,opacity=0.2] (\x*\dist,0) -- (\x*\dist+\w,\inh) -- (\x*\dist+\w,\inh+\h) -- (\x*\dist,\h) -- cycle;

\node at (\x*\dist + 0.87 * \w, \inh+ \h - 0.3) {$G_{\x}$};
\node at (\x*\dist + 0.6 * \w, 0) {$t_{\x}=\x$};
\draw[-] (\x*\dist + 0.6 * \w, \tline+0.05) -- (\x*\dist + 0.6 * \w, \tline-0.05);

\node (A\x) at (\x*\dist+\Ax, 0+\Ay) {};
\node (B\x) at (\x*\dist+\Bx, 0+\By) {};
\node (C\x) at (\x*\dist+\Cx, 0+\Cy) {};
\node (D\x) at (\x*\dist+\Dx, 0+\Dy) {};
\node (E\x) at (\x*\dist+\Ex, 0+\Ey) {};
}

\draw[thick] (A1) -- (D1);
\draw[thick] (A1) -- (C1);
\draw[thick] (B1) -- (C1);
\draw[thick] (B1) -- (E1);
\draw[thick] (A2) -- (D2);
\draw[thick] (A2) -- (C2);
\draw[thick] (B2) -- (C2);
\draw[thick] (B2) -- (E2);

\node[dot=9pt,draw=black, fill=mywhite] at (A1) {$w_1$};
\node[dot=9pt,draw=black, fill=mywhite] at (B1) {$u_1$};
\node[dot=9pt,draw=black, fill=mywhite] at (C1) {$v$};
\node[dot=9pt,draw=black, fill=mywhite] at (D1) {$w_2$};
\node[dot=9pt,draw=black, fill=myblue] at (E1) {$u_2$};
\node[dot=9pt,draw=black, fill=mywhite] at (A2) {$w_1$};
\node[dot=9pt,draw=black, fill=myred] at (B2) {$u_1$};
\node[dot=9pt,draw=black, fill=mywhite] at (C2) {$v$};
\node[dot=9pt,draw=black, fill=mywhite] at (D2) {$w_2$};
\node[dot=9pt,draw=black, fill=mywhite] at (E2) {$u_2$};
\end{tikzpicture}
    \end{minipage}%
    \begin{minipage}{0.4\textwidth}
    \centering
            \begin{tikzpicture}[
dot/.style = {draw, circle, minimum size=0.4cm,
              inner sep=0pt, outer sep=0pt},
dot/.default = 6pt
]
\scriptsize
\pgfmathsetmacro{\tline}{-0.25}
\pgfmathsetmacro{\w}{1.8}
\pgfmathsetmacro{\h}{2.0}
\pgfmathsetmacro{\inh}{0.7}
\pgfmathsetmacro{\dist}{1.6}
\pgfmathsetmacro{\Ax}{0.5}
\pgfmathsetmacro{\Ay}{1.4}
\pgfmathsetmacro{\Bx}{0.6}
\pgfmathsetmacro{\By}{0.6}
\pgfmathsetmacro{\Cx}{1.2}
\pgfmathsetmacro{\Cy}{1.4}
\pgfmathsetmacro{\Dx}{1.1}
\pgfmathsetmacro{\Dy}{2}
\pgfmathsetmacro{\Ex}{1.3}
\pgfmathsetmacro{\Ey}{0.8}


\draw[->] (1.8,\tline) -- (5.1,\tline);
\node at (3.5,\tline-0.3) {$(\TG',v')$};

\foreach \x in {1,...,2}
{
\draw[fill=gray!90!black,opacity=0.2] (\x*\dist,0) -- (\x*\dist+\w,\inh) -- (\x*\dist+\w,\inh+\h) -- (\x*\dist,\h) -- cycle;

\node at (\x*\dist + 0.87 * \w, \inh+ \h - 0.3) {$G_{\x}'$};
\node at (\x*\dist + 0.6 * \w, 0) {$t_{\x}=\x$};
\draw[-] (\x*\dist + 0.6 * \w, \tline+0.05) -- (\x*\dist + 0.6 * \w, \tline-0.05);

\node (A\x) at (\x*\dist+\Ax, 0+\Ay) {};
\node (B\x) at (\x*\dist+\Bx, 0+\By) {};
\node (C\x) at (\x*\dist+\Cx, 0+\Cy) {};
\node (D\x) at (\x*\dist+\Dx, 0+\Dy) {};
\node (E\x) at (\x*\dist+\Ex, 0+\Ey) {};
}

\draw[thick] (A1) -- (D1);
\draw[thick] (A1) -- (C1);
\draw[thick] (B1) -- (C1);
\draw[thick] (B1) -- (E1);
\draw[thick] (A2) -- (D2);
\draw[thick] (A2) -- (C2);
\draw[thick] (B2) -- (C2);
\draw[thick] (B2) -- (E2);

\node[dot=9pt,draw=black, fill=mywhite] at (A1) {$w_1'$};
\node[dot=9pt,draw=black, fill=mywhite] at (B1) {$u_1'$};
\node[dot=9pt,draw=black, fill=mywhite] at (C1) {$v'$};
\node[dot=9pt,draw=black, fill=myblue] at (D1) {$w_2'$};
\node[dot=9pt,draw=black, fill=mywhite] at (E1) {$u_2'$};
\node[dot=9pt,draw=black, fill=mywhite] at (A2) {$w_1'$};
\node[dot=9pt,draw=black, fill=myred] at (B2) {$u_1'$};
\node[dot=9pt,draw=black, fill=mywhite] at (C2) {$v'$};
\node[dot=9pt,draw=black, fill=mywhite] at (D2) {$w_2'$};
\node[dot=9pt,draw=black, fill=mywhite] at (E2) {$u_2'$};
\end{tikzpicture}
    \end{minipage}
        \caption{Counterexample, used in Theorem~\ref{th:TandGweak}; colour $c_1$ is denoted by a red filling (node $u_1$ in $G_2$ and node $u_1'$ in $G_2'$) and 
        colour $c_2$ is denoted by a blue filling (node $u_2$ in $G_1$ and node $w_2'$ in $G_1'$)}
    \label{fig:tandgweak_ex}
\end{figure}
First, we examine
time-and-graph TGNNs \cite{GR22}, denoted by $\tandgTGNN[\mathcal{M}, \mathcal{C}]$ and parametrised by a class of MPNNs $\mathcal{M}$ and 
a class of cell functions $\mathit{C}$. These TGNNs differ from the previously discussed classes $\idTGNN[\mathcal{M}]$, 
not in the specific form of MPNNs used, but in how temporal information is processed. 
In particular, $\mathcal{T}_{\tandg}$ appears less powerful at processing the combined information 
of a current snapshot $(G_i, t_i)$ 
and the embedded information of the past snapshots $(G_1, t_1), \dotsc, (G_{i-1}, t_{i-1})$.
\begin{restatable}{theorem}{tandgweak}
    \label{th:TandGweak}
    $\PTLK \not\leq \tandgTGNN[\mathcal{M},\mathcal{C}]$, for all  $\mathcal{M}$ and $\mathcal{C}$.    
\end{restatable}
\begin{proof}[Proof sketch.]
    Consider the $\PTLK$ formula $\varphi = \Diamond(c_1 \land \prev{}\Diamond c_2)$, satisfied by
    all pointed temporal graphs $(\TG, (v,t_n))$ where $n \geq 2$ and node $v$ has a
    neighbour that is of colour $c_1$ at timestamp $t_n$ and which has a neighbour of colour $c_2$ at timestamp $t_{n-1}$. For example, see Figure~\ref{fig:tandgweak_ex}, where $(\TG,v)$ satisfies $\varphi$ and $(\TG',v')$ does not. 
    We can show that there exists no TGNN $\model = (M_1, M_2, \cell) \in \tandgTGNN[\mathcal{M},\mathcal{C}]$ for any class of  MPNNs $\mathcal{M}$ and cell functions
    $\mathcal{C}$ that captures $\varphi$. 
    Intuitively,  $M_1$ handles the present label information
    and $M_2$ the past information, but neither handles both. 
    While  $\cell$ can use outputs of $M_1$ and $M_2$, it has no access to the topology of the static graph. 
    However, to check whether $\varphi$ is satisfied, this is necessary. 
    This results in the fact that $\model$ either accepts both $(\TG,v)$ and $(\TG',v')$ of Figure~\ref{fig:tandgweak_ex}, or none.
    A formal proof of this is provided in Appendix~\ref{proof:tandgweak}.
\end{proof}

The natural next question is to identify a fragment of $\PTLK$ whose  formulas can be expressed by time-and-graph TGNNs.
Let $\mathcal{F} \subset \FNN{\trelu}$ be the class of single layer FNNs with truncated-ReLU activations.
We let
$\mathcal{L}_1$
be the fragment of $\PTLK$ containing only formulas $\varphi$ such that for all 
$\Diamond \psi \in \sub(\varphi)$ there is either no $Q\chi \in \sub(\psi)$ with $Q \in \{\prev{}, \past{}\}$ 
or for all $c \in \sub(\psi)$ there is $Q\chi \in \sub(\psi)$ with $Q \in \{\prev{}, \past{}\}$ such that 
$c \in \sub(\chi)$.
Hence, $\mathcal{L}_1$ restricts the allowed interaction of operators in $\PTLK$.
For example $\Diamond \past c_1$ is a formula of $\mathcal{L}_1$, but $\Diamond(\past{}c_1 \land c_2)$ is not.

\begin{restatable}{theorem}{tandg}
    \label{sec:variants_logic2tnn;thm:tandg} 
    $\mathcal{L}_1 \leq \tandgTGNN[\hat{\mathcal{M}},\mathcal{F}]$.
\end{restatable} 
\begin{proof}[Proof sketch]
    We apply the inductive approach of Theorems \ref{sec:logic2tgnn;thm:ptlk} and \ref{sec:log2tgnn;thm:benediktnunn}
    to formulas $\varphi \in \mathcal{L}_1$. We construct time-and-graph TGNNs
    $\model_\varphi = (M_1, M_2, \cell, \out) \in \tandgTGNN[\hat{\mathcal{M}},\mathcal{F}]$ where the three components
    partition the evaluations according to the syntactic form of the subformulas $\varphi_i$ of $\varphi$.
    The MPNN $M_1$ captures all subformulas $\varphi_i$ that do not include the temporal operators $\prev{}$ or $\past{}$, meaning
    it processes static properties of the current snapshot.
    Similarly, the MPNN $M_2$ handles all subformulas where every atomic subformula is nested under a temporal operator,
    indicating that the formula evaluation relies solely on the prior hidden states.
    The cell function $\cell$ manages the remaining subformulas, which must be Boolean subformulas
    ($\neg \psi_1$ or $\psi_1 \land \psi_2$) due to the definition of $\mathcal{L}_1$.
    Correctness is established through a nested induction over time and subformula structure, as done 
    previously. 
    The structure of $\mathcal{L}_1$ ensures that all subformulas are appropriately covered by the
    division among $M_1$, $M_2$, and $\cell$. A full proof is provided in Appendix~\ref{proof:tandg}.
\end{proof}
We remark that some time-and-graph TGNNs considered in the literature \cite{GR22} utilise a gated recurrent unit (GRU) \cite{ChoMGBBSB14, ChungGCB14}
as their cell function. 
The class of TGNNs considered in \Cref{sec:variants_logic2tnn;thm:tandg} is  a subset of this broader class.
There are also  similar, \emph{time-then-graph} TGNNs, considered in the literature \cite{GR22}. 
They first process the temporal information of a temporal graph
using a recurrent neural network (RNN), followed by the application of an MPNN to
capture the topology. 
It is shown that for each
time-and-graph TGNN, there exists an equivalent time-then-graph TGNN
\cite[Theorems~3.5 and 3.6]{GR22}. Thus, our results from Theorem~\ref{sec:variants_logic2tnn;thm:tandg} 
transfer to time-then-graph TGNNs.

Next we consider global TGNNs, $\globTGNN[\mathcal{M}, \mathcal{Q}, \circ]$, recently studied in the literature \cite{WR24}.
They are parametrised by a class of MPNNs $\mathcal{M}$, time functions from $\mathcal{Q}$, and an operation $\circ$ combining label and temporal information in the aggregation. 
Similarly to time-and-graph TGNNs, $\globTGNN$ cannot capture all properties definable in $\PTLK$.
\begin{restatable}{theorem}{globTGNNweak}
    \label{thm:globTGNNweak}
    $\PTLK \not\leq \globTGNN[\mathcal{M}, \mathcal{Q}, \circ]$, for all  $\mathcal{M}$, $\mathcal{Q}$, and  $\circ$.    
\end{restatable}
\begin{proof}[Proof sketch.]
    Consider the $\PTLK$ formula $\varphi = \prev{} c_1$, which is satisfied by all
    pointed temporal graphs $(\TG, (v,t_n))$ where $n \geq 2$ and the node $v$
    was of colour $c_1$ at timestamp $t_{n-1}$. 
    Interestingly, there is no
    $\globTGNN[\mathcal{M}, \mathcal{Q}, \circ]$ for any  $\mathcal{M}$,
    $\mathcal{Q}$, and $\circ$ that can express this property. The reason is as follows:
    the notion of temporal neighbourhood utilised by global TGNNs does not give a 
    node $v$ access to its own past information. It can only access past informations 
    of neighbours.
    Since this is the only form of 
    temporal information used by such TGNNs, it excludes the ability to recognise whether $v$ itself was coloured $c_1$
    in the past.
    For the full proof, see Appendix~\ref{proof:globtgnnweak}.
\end{proof}

Next, we turn to identifying a fragment of $\PTLK$ captured by $\globTGNN$. The architecture of global TGNNs appears to necessitate a class of MPNN components
that effectively utilise temporal information. Specifically, we use MPNNs incorporating a learnable message
function $\msg$ in their aggregation. This enables global TGNNs to process messages
based on temporal information prior to aggregation. Thus, we use $\hat{\mathcal{M}}_\msg$ (see Section~\ref{sec:tgnn} to recall the definition of such GNNs).
Concerning $\phi$ and $\circ$, meaning the operations by which temporal information is processed 
and then combined with static neighbourhood 
information during aggregation, we employ time2vec \cite{Kazemi2019} and concatenation 
$\con$ of vectors, akin to models such as TGAT or Temporal Graph Sum \cite{rossi2020temporal}. 
Let $\mathcal{Q}_\mathsf{time2vec}$ be the class of all time2vec functions (see Appendix~\ref{def:time2vec} for a formal definition). 
We let
$\mathcal{L}_2$ be the fragment of $\PTLK$, whose formulas $\varphi$ are such that 
 for all $Q\psi \in \sub(\varphi)$ with $Q \in \{\prev{},\past{}\}$ we have $\psi = \Diamond\chi$ for some formula $\chi$.
For example $ \past{} \Diamond c_1$ is  a formula of $\mathcal{L}_2$, but $ \past c_1$ is not.

\begin{restatable}{theorem}{globaltgnn}
    \label{sec:variants_logic2tnn;thm:globaltgnn}
    $\mathcal{L}_2 \leq \globTGNN[\hat{\mathcal{M}}_\msg, \mathcal{Q}_\mathsf{time2vec}, \con]$.
\end{restatable}
\begin{proof}[Proof sketch]
    We use the inductive approach from Theorem~\ref{sec:logic2tgnn;thm:ptlk}, but now use a global TGNN
    $\model_\varphi = (M, \phi, \out)$. 
    For each subformula $\varphi_i$, a layer $l^{(i)}$ in $M$ computes its semantics.
    Boolean subformulas are handled as in previous results and temporal subformulas $\prev{}\psi$ and $\past{}\psi$ are deferred,
    since their semantics are captured indirectly through $\Diamond$ subformulas, as ensured by the definition of $\mathcal{L}_2$.
    The key mechanism to do this is the interaction between $\phi$ and $\msg$: for each $Q\Diamond \psi$ with $Q \in \{\past{},\prev{}\}$, $\msg$ selectively passes
    information based on $\phi(t)$, enforcing that only the correct previous or past timepoints contribute.
    If $\psi$ is of the form $\prev{}\Diamond \chi$ or $\past{} \Diamond \chi$, $\msg$ filters messages from exactly the previous timepoint
    or from all earlier timepoints, respectively. Otherwise, $\msg$ restricts aggregation to the current timepoint.
    Correctness follows by induction over time and subformula structure. A full proof is provided in Appendix~\ref{proof:globtgnn}.
\end{proof}
While employing time2vec functions is motivated by common practice, it is not essential for achieving the previous result.
In fact, it suffices to have a function $\phi$ capable of mapping the values $0$, $-1$, and $t$, where $t \leq -2$,
to distinct values.
Similarly, $\con$ is not strictly necessary. The same result can be achieved, with a slightly adapted construction,
using entrywise multiplication or addition as $\circ$.

Now, combining the previous results, namely Theorems~\ref{sec:logic2tgnn;thm:ptlk} to \ref{sec:variants_logic2tnn;thm:globaltgnn}, 
and the insight that $\Diamond(c_1 \land \prev{}\Diamond c_2) \in \mathcal{L}_2$ and $\prev{}c_1 \in \mathcal{L}_1$ we immediately
get the following inexpressiveness results. Let $\mathcal{T}_1$ and $\mathcal{T}_2$ be two classes of TGNNs. We let $\mathcal{T}_1 \leq \mathcal{T}_2$
if for all $T \in \mathcal{T}_1$ there is $T' \in \mathcal{T}_2$ such that 
for all discrete, pointed temporal graphs $(\TG,v)$ we have that $T(\TG,v) = 1$ if and only if $T'(\TG,v) = 1$. 
Accordingly, we define $\mathcal{T}_1 \not\leq \mathcal{T}_2$ if $\mathcal{T}_1 \leq \mathcal{T}_2$ does not hold, and we define  $\mathcal{T}_1 \not\equiv \mathcal{T}_2$ if $\mathcal{T}_1 \not\leq \mathcal{T}_2$ or $\mathcal{T}_2 \not\leq \mathcal{T}_1$.
\begin{corollary}
    \label{cor:exprrelation}
    $\idTGNN[\hat{\mathcal{M}}] \not\leq \tandgTGNN[\hat{\mathcal{M}},\mathcal{F}]$, $\idTGNN[\hat{\mathcal{M}}] \not\leq \globTGNN[\hat{\mathcal{M}}_\msg, \mathcal{Q}_\mathsf{time2vec}, \con]$, and 
    $\tandgTGNN[\hat{\mathcal{M}},\mathcal{F}] \not\equiv \globTGNN[\hat{\mathcal{M}}_\msg, \mathcal{Q}_\mathsf{time2vec}, \con]$.
\end{corollary}

\begin{figure}
\tikzset{every picture/.style={line width=0.75pt}} 

\begin{tikzpicture}[x=0.75pt,y=0.75pt,yscale=-0.8,xscale=0.8, every node/.style={font=\scriptsize}]


\draw (295,158.4) node [anchor=north,name=ptlk] [inner sep=0.75pt]    {$\mathsf{PTL}_{\mathsf{P,Y}} \times \ \mathsf{K}$};
\draw (295.5,55.6) node [anchor=south,name=Trec] [inner sep=0.75pt]    {$\mathcal{T}_{\mathsf{rec}}[\hat{\mathcal{M}}]$};
\draw (95,184.4) node [anchor=north west, name=TandG][inner sep=0.75pt]    {$\mathcal{T}_{\mathsf{TandG}}[\hat{\mathcal{M}} ,\ \mathcal{F}]$};
\draw (223,252.4) node [anchor=north east, name=Lone] [inner sep=0.75pt]    {$\mathcal{L}_{1}$};
\draw (372,253.4) node [anchor=north west, name=Ltwo][inner sep=0.75pt]    {$\mathcal{L}_{2}$};
\draw (381,184.4) node [anchor=north west, name=Tglob][inner sep=0.75pt]    {$\mathcal{T}_{\mathsf{glob}}[\hat{\mathcal{M}}_{\mathsf{msg}} ,Q_{\mathsf{time2vec}} ,||]$};

\draw (495,90) node [anchor=north, name=Khash] [inner sep=0.75pt] {$\mathsf{PTL}_{\mathsf{P,Y}} \times \ K^{\#}$};
\draw (495,55.6) node [anchor=south, name=TrecKhash] [inner sep=0.75pt]  {$\mathcal{T}_{\mathsf{rec}}[\hat{\mathcal{M}}_{K^{\#}}]$};
\draw (695,90) node [anchor=north, name=LMP] [inner sep=0.75pt]  {$\mathsf{PTL}_{\mathsf{P,Y}} \times \ (\mathcal{L} -\mathsf{MP})^{2}$};
\draw (695,55.6) node [anchor=south, name=TrecLMP] [inner sep=0.75pt]   {$\mathcal{T}_{\mathsf{rec}}[\mathcal{OL}\mathit{trReLU} -\text{GNN}]$};
\draw (390.79,128.05) node [anchor=south] [inner sep=0.75pt]   [rotate=20]{$\varsubsetneq $};
\draw (490.79,151.05) node [anchor=south] [inner sep=0.75pt]   [rotate=10]{$\varsubsetneq $};

\draw (338.79,238.05) node [anchor=south] [inner sep=0.75pt]  [rotate=-55,xscale=-1]  {$\varsubsetneq $};
\draw (251.21,238.05) node [anchor=south] [inner sep=0.75pt]  [rotate=-305]  {$\varsubsetneq $};
\draw (202.6,125.6) node [anchor=south] [inner sep=0.75pt]  [rotate=-315]  {$\leq $};
\draw (405.6,226.6) node [anchor=south] [inner sep=0.75pt]  [rotate=-315]  {$\leq $};
\draw (177.6,226.6) node [anchor=south] [inner sep=0.75pt]  [rotate=-40,xscale=-1]  {$\leq $};
\draw (298.31,106) node [anchor=north] [inner sep=0.75pt]  [rotate=-270]  {$\leq $};
\draw (293.91,107) node [anchor=east] [inner sep=0.75pt]   [align=left] {Th. 1};
\draw (172.7,235.93) node [anchor=north] [inner sep=0.75pt]  [rotate=-40] [align=left] {Th. 4};
\draw (417.12,235.93) node [anchor=north] [inner sep=0.75pt]  [rotate=-315] [align=left] {Th. 6};
\draw (209.12,128.6) node [anchor=north] [inner sep=0.75pt]  [rotate=-315] [align=left] {Th. 7};
\draw (215.72,175.12) node [anchor=south east] [inner sep=0.75pt]  [xscale=-1]  {$\not\leq $};
\draw (213,182) node [anchor=north west][inner sep=0.75pt]  [align=left] {Th. 3};
\draw (373,182) node [anchor=north east] [inner sep=0.75pt] [align=left] {Th. 5};
\draw (374.28,175.12) node [anchor=south east] [inner sep=0.75pt]   {$\nleq $};

\draw (498.31,70) node [anchor=north] [inner sep=0.75pt]  [rotate=-270]  {$\leq $};
\draw (493.91,71) node [anchor=east] [inner sep=0.75pt]   [align=left] {Th. 2};
\draw (698.31,70) node [anchor=north] [inner sep=0.75pt]  [rotate=-270]  {$\leq $};
\draw (693.91,71) node [anchor=east] [inner sep=0.75pt]   [align=left] {Th. 2};

\draw[-] (Trec) -- (ptlk);
\draw[-] (Trec) -- (TandG);
\draw[-] (ptlk) -- (TandG);
\draw[-] (TandG) -- (Lone);
\draw[-] (ptlk) -- (Lone);
\draw[-] (ptlk) -- (Ltwo);
\draw[-] (ptlk) -- (Tglob);
\draw[-] (Tglob) -- (Ltwo);

\draw[-] (Khash) -- (ptlk);
\draw[-] (LMP) -- (ptlk);
\draw[-] (Khash) -- (TrecKhash);
\draw[-] (LMP) -- (TrecLMP);

\end{tikzpicture}
    \caption{An overview of our expressive power results}
    \label{fig:overview}
\end{figure}
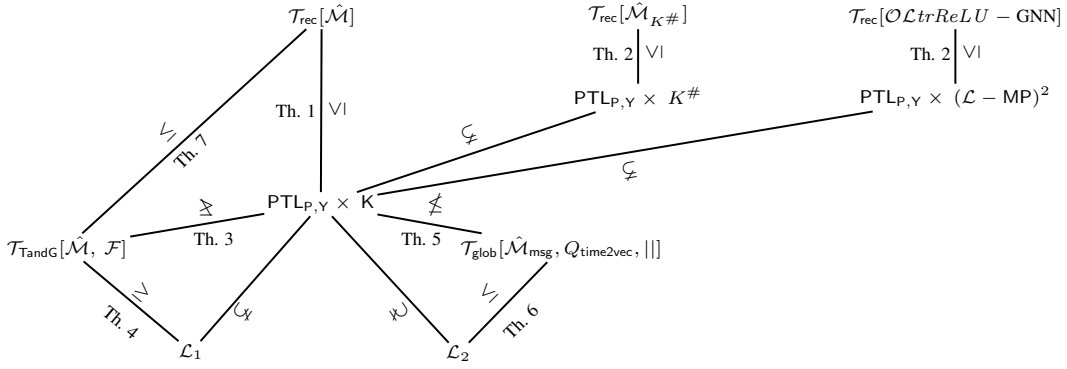


To complement the picture of expressiveness relationships obtained so far, 
we  establish the following relationship between time-and-graph and recursive TGNNs;
recall that  $\mathcal{F} \subseteq \FNN{\trelu}$ is the class of single layer FNNs with 
truncated-ReLU activations.
\begin{restatable}{theorem}{idstrongest}
    \label{thm:idstrongest}
$\tandgTGNN[\hat{M}, \mathcal{F}] \leq \idTGNN[\hat{M}]$.
\end{restatable}
\begin{proof}
    Let $T \in \tandgTGNN[\hat{M}, \mathcal{F}]$ with $T = (M_1, M_2, \cell, \out)$, where each $M_i$ is of input
    dimensionality $m_i$ and output dimensionality $n_i$, and $\cell$ is of input dimensionality
    $n_1+n_2$ and output dimensionality $n_3$. The construction of the witness $T' \in \idTGNN[\hat{M}]$ is straightforward: $T'$ is given 
    by $(M_3, \out)$, where $M_3$ is given by  $\cell \circ M_1 || M_2$, which represents 
    the MPNN of input dimensionality $m_1 + m_2$ and output dimensionality $n_3$ simulating 
    $M_1$ on the first $m_1$ input dimensions, $M_2$ on the last $m_2$ input dimensions, and then applies a
    message-passing layer, where the combination function ignores input from $\agg$ and uses $\cell$ otherwise.
\end{proof}
Figure~\ref{fig:overview} summarises the expressive results we have established in the paper.





\section{Conclusion}
\label{sec:outlook}

We have initiated the study of the expressive capabilities of TGNNs  using logical languages,
which is motivated by  successful logical characterisation of static GNNs.
As we have showed, product logics combining temporal and modal logics, are particularly well-suited to achieve this goal.
In particular, we have studied three classes of TGNNs: recursive, time-and-graph, and global TGNNs.
We have showed that 
recursive TGNNs can express all properties definable in the product logic $\PTLK$, combining the standard temporal and modal logics.
Moreover we have obtained analogous results by relating variants of $\PTLK$ to recursive TGNNs that exploit specific classes of MPNNs.
In contrast, neither time-and-graph or global TGNNs can express all properties definable in $\PTLK$.
The reason is that $\PTLK$ allows for arbitrary interaction between logical operators expressing temporal and spatial properties.
By restricting the form of these interaction, we have obtained fragments of $\PTLK$ which can be captured by time-and-graph or global TGNNs, respectively.

Our results provide  new insights into the expressive power of TGNNs and show that  TGNN architectures significantly differ on the spatio-temporal properties they can capture.
Better understanding of TGNN capabilities is crucial for choosing appropriate models for a downstream task and help in developing more powerful architectures.
Since this is the first work on the logical expressiveness of TGNNs, it has a number of interesting next steps that can be performed in future. 
Among others, we plan to investigate tight expressive bounds, and for a larger amount of TGNN architectures. Furthermore, we plan to transfer results known from research on product logics to TGNNs, including computational complexity analysis and finite model theory.

\paragraph{Limitations.} 
The results established here are of a strictly formal nature.
Our expressive results focus on showing which classes of TGNNs can express (i.e. detect) temporal graph properties expressible in particular logics. 
Hence our results do not aim to show which properties can be learnt in practice, but to show fundamental relations between expressiveness of TGNNs and product logics.



\begin{ack}
This research is partially funded by the European Union (ERC, ExtenDD, project number: 101054714). Views and opinions expressed are however those of the authors only and do not necessarily reflect those of the European Union or the European Research Council. Neither the European Union nor the granting authority can be held responsible for them.
\end{ack}


\bibliographystyle{abbrv}
\bibliography{tgnnbiblio}

\begin{thebibliography}{10}

\bibitem{ahvonen2024logical}
V.~Ahvonen, D.~Heiman, A.~Kuusisto, and C.~Lutz.
\newblock Logical characterizations of recurrent graph neural networks with reals and floats.
\newblock In {\em 38th Conference on Neural Information Processing Systems}, 2024.

\bibitem{DBLP:journals/amai/ArtaleF00}
A.~Artale and E.~Franconi.
\newblock A survey of temporal extensions of description logics.
\newblock {\em Ann. Math. Artif. Intell.}, 30(1-4):171--210, 2000.

\bibitem{BarceloKM0RS20}
P.~Barcel{\'{o}}, E.~V. Kostylev, M.~Monet, J.~P{\'{e}}rez, J.~L. Reutter, and J.~P. Silva.
\newblock The logical expressiveness of graph neural networks.
\newblock In {\em 8th Conference on Learning Representations, {ICLR}}, 2020.

\bibitem{DBLP:conf/iclr/BarceloKLP24}
P.~Barcel{\'{o}}, A.~Kozachinskiy, A.~W. Lin, and V.~V. Podolskii.
\newblock Logical languages accepted by transformer encoders with hard attention.
\newblock In {\em 12th Conference on Learning Representations, {ICLR}}, 2024.

\bibitem{BenediktLMT24}
M.~Benedikt, C.~Lu, B.~Motik, and T.~Tan.
\newblock Decidability of graph neural networks via logical characterizations.
\newblock In {\em 51st Colloquium on Automata, Languages, and Programming, {ICALP}}, LIPIcs. Schloss Dagstuhl - Leibniz-Zentrum f{\"{u}}r Informatik, 2024.

\bibitem{DBLP:journals/apin/BennettCWZ02}
B.~Bennett, A.~G. Cohn, F.~Wolter, and M.~Zakharyaschev.
\newblock Multi-dimensional modal logic as a framework for spatio-temporal reasoning.
\newblock {\em Appl. Intell.}, 17(3):239--251, 2002.

\bibitem{ChenWX22}
J.~Chen, X.~Wang, and X.~Xu.
\newblock {GC-LSTM:} graph convolution embedded {LSTM} for dynamic network link prediction.
\newblock {\em Appl. Intell.}, 52(7):7513--7528, 2022.

\bibitem{chen2023calibrate}
Y.~Chen and D.~Wang.
\newblock Calibrate and boost logical expressiveness of {GNN} over multi-relational and temporal graphs.
\newblock In {\em 37th Conference on Neural Information Processing Systems}, 2023.

\bibitem{ChoMGBBSB14}
K.~Cho, B.~van Merrienboer, {\c{C}}.~G{\"{u}}l{\c{c}}ehre, D.~Bahdanau, F.~Bougares, H.~Schwenk, and Y.~Bengio.
\newblock Learning phrase representations using {RNN} encoder-decoder for statistical machine translation.
\newblock In {\em 2014 Conference on Empirical Methods in Natural Language Processing, {EMNLP}}, pages 1724--1734. {ACL}, 2014.

\bibitem{ChungGCB14}
J.~Chung, {\c{C}}.~G{\"{u}}l{\c{c}}ehre, K.~Cho, and Y.~Bengio.
\newblock Empirical evaluation of gated recurrent neural networks on sequence modeling.
\newblock {\em CoRR}, abs/1412.3555, 2014.

\bibitem{Cini_Marisca_Bianchi_Alippi_2023}
A.~Cini, I.~Marisca, F.~M. Bianchi, and C.~Alippi.
\newblock Scalable spatiotemporal graph neural networks.
\newblock {\em 37th AAAI Conference on Artificial Intelligence}, 37(6):7218--7226, 2023.

\bibitem{grau2025correspondence}
B.~Cuenca~Grau, E.~Feng, and P.~A. Wa{\l}{\k{e}}ga.
\newblock The correspondence between bounded graph neural networks and fragments of first-order logic.
\newblock {\em arXiv preprint arXiv:2505.08021}, 2025.

\bibitem{demri2024computational}
S.~Demri and P.~A. Wa{\l}{\k{e}}ga.
\newblock Computational complexity of standpoint {LTL}.
\newblock {\em arXiv preprint arXiv:2408.08557}, 2024.

\bibitem{GR22}
J.~Gao and B.~Ribeiro.
\newblock On the equivalence between temporal and static equivariant graph representations.
\newblock In {\em 39th Conference on Machine Learning}, volume 162 of {\em Proceedings of Machine Learning Research}, pages 7052--7076. PMLR, 2022.

\bibitem{DBLP:conf/kr/GiganteAL23}
N.~Gigante, L.~{Gómez Álvarez}, and T.~S. Lyon.
\newblock Standpoint linear temporal logic.
\newblock In {\em Proceedings of the 20th International Conference on Principles of Knowledge Representation and Reasoning, {KR}}, pages 311--321, 2023.

\bibitem{GilmerSRVD17}
J.~Gilmer, S.~S. Schoenholz, P.~F. Riley, O.~Vinyals, and G.~E. Dahl.
\newblock Neural message passing for quantum chemistry.
\newblock In {\em 34th Conference on Machine Learning, {ICML}}, Proceedings of Machine Learning Research. {PMLR}, 2017.

\bibitem{Grohe21}
M.~Grohe.
\newblock The logic of graph neural networks.
\newblock In {\em 36th {ACM/IEEE} Symposium on Logic in Computer Science, {LICS}}, pages 1--17. {IEEE}, 2021.

\bibitem{HajiramezanaliH19}
E.~Hajiramezanali, A.~Hasanzadeh, K.~R. Narayanan, N.~Duffield, M.~Zhou, and X.~Qian.
\newblock Variational graph recurrent neural networks.
\newblock In {\em 32nd Conference on Neural Information Processing Systems, NeurIPS}, pages 10700--10710, 2019.

\bibitem{DBLP:journals/jcss/HalpernV89}
J.~Y. Halpern and M.~Y. Vardi.
\newblock The complexity of reasoning about knowledge and time. {I}. lower bounds.
\newblock {\em J. Comput. Syst. Sci.}, 38(1):195--237, 1989.

\bibitem{hauke2025aggregate}
S.~P. Hauke and P.~A. Wa{\l}{\k{e}}ga.
\newblock Aggregate-combine-readout {GNNs} are more expressive than logic {C2}.
\newblock {\em arXiv preprint arXiv:2508.06091}, 2025.

\bibitem{Kapoor20}
A.~Kapoor, X.~Ben, L.~Liu, B.~Perozzi, M.~Barnes, M.~Blais, and S.~O'Banion.
\newblock Examining {COVID-19} forecasting using spatio-temporal graph neural networks.
\newblock {\em CoRR}, abs/2007.03113, 2020.

\bibitem{Kazemi2019}
S.~M. Kazemi, R.~Goel, S.~Eghbali, J.~Ramanan, J.~Sahota, S.~Thakur, S.~Wu, C.~Smyth, P.~Poupart, and M.~A. Brubaker.
\newblock Time2vec: Learning a vector representation of time.
\newblock {\em CoRR}, abs/1907.05321, 2019.

\bibitem{Kripke1972-KRINAN-2}
S.~A. Kripke.
\newblock {\em Naming and Necessity}.
\newblock Harvard Univ.\ Press, 1972.

\bibitem{kurucz2003many}
A.~Kurucz, F.~Wolter, M.~Zakharyaschev, and D.~M. Gabbay.
\newblock {\em Many-dimensional modal logics: Theory and applications}, volume 148.
\newblock Elsevier, 2003.

\bibitem{LiHCSWZP19}
J.~Li, Z.~Han, H.~Cheng, J.~Su, P.~Wang, J.~Zhang, and L.~Pan.
\newblock Predicting path failure in time-evolving graphs.
\newblock In {\em 25th {ACM} {SIGKDD} Conference on Knowledge Discovery {\&} Data Mining, {KDD}}. {ACM}, 2019.

\bibitem{LongaLSBLLSP23}
A.~Longa, V.~Lachi, G.~Santin, M.~Bianchini, B.~Lepri, P.~Lio, F.~Scarselli, and A.~Passerini.
\newblock Graph neural networks for temporal graphs: State of the art, open challenges, and opportunities.
\newblock {\em Trans. Mach. Learn. Res.}, 2023, 2023.

\bibitem{Luo22a}
Y.~Luo and P.~Li.
\newblock Neighborhood-aware scalable temporal network representation learning.
\newblock In {\em 1st Learning on Graphs Conference}, volume 198 of {\em Proceedings of Machine Learning Research}. PMLR, 2022.

\bibitem{DBLP:conf/time/LutzWZ08}
C.~Lutz, F.~Wolter, and M.~Zakharyaschev.
\newblock Temporal description logics: {A} survey.
\newblock In S.~Demri and C.~S. Jensen, editors, {\em 15th Symposium on Temporal Representation and Reasoning, {TIME}}, pages 3--14. {IEEE} Computer Society, 2008.

\bibitem{DBLP:books/daglib/0093551}
M.~Marx and Y.~Venema.
\newblock {\em Multi-dimensional modal logic}, volume~4 of {\em Applied logic series}.
\newblock Kluwer, 1997.

\bibitem{MicheliT22}
A.~Micheli and D.~Tortorella.
\newblock Discrete-time dynamic graph echo state networks.
\newblock {\em Neurocomputing}, 496:85--95, 2022.

\bibitem{morris2019weisfeiler}
C.~Morris, M.~Ritzert, M.~Fey, W.~L. Hamilton, J.~E. Lenssen, G.~Rattan, and M.~Grohe.
\newblock {W}eisfeiler and {l}eman go neural: Higher-order graph neural networks.
\newblock In {\em 33rd Conference on Artificial Intelligence, {AAAI}}, volume~33, pages 4602--4609, 2019.

\bibitem{NunnSST24}
P.~Nunn, M.~S{\"{a}}lzer, F.~Schwarzentruber, and N.~Troquard.
\newblock A logic for reasoning about aggregate-combine graph neural networks.
\newblock In {\em 33rd Joint Conference on Artificial Intelligence, {IJCAI}}, pages 3532--3540. ijcai.org, 2024.

\bibitem{ParejaDCMSKKSL20}
A.~Pareja, G.~Domeniconi, J.~Chen, T.~Ma, T.~Suzumura, H.~Kanezashi, T.~Kaler, T.~B. Schardl, and C.~E. Leiserson.
\newblock Evolvegcn: Evolving graph convolutional networks for dynamic graphs.
\newblock In {\em 34th Conference on Artificial Intelligence, {AAAI}}, pages 5363--5370. {AAAI} Press, 2020.

\bibitem{Pnueli:1977}
A.~Pnueli.
\newblock The temporal logic of programs.
\newblock In {\em 18th Symp. on Foundations of Computer Science, {FOCS}}, pages 46--57. IEEE, 1977.

\bibitem{rossi2020temporal}
E.~Rossi, B.~Chamberlain, F.~Frasca, D.~Eynard, F.~Monti, and M.~Bronstein.
\newblock Temporal graph networks for deep learning on dynamic graphs.
\newblock {\em arXiv preprint arXiv:2006.10637}, 2020.

\bibitem{SL22}
M.~S{\"{a}}lzer and M.~Lange.
\newblock Reachability in simple neural networks.
\newblock {\em Fundam. Informaticae}, 189(3-4):241--259, 2022.

\bibitem{SST25}
M.~S{\"{a}}lzer, F.~Schwarzentruber, and N.~Troquard.
\newblock Verifying quantized graph neural networks is {PSPACE}-complete.
\newblock {\em CoRR}, abs/2502.16244, 2025.

\bibitem{SankarWGZY20}
A.~Sankar, Y.~Wu, L.~Gou, W.~Zhang, and H.~Yang.
\newblock Dysat: Deep neural representation learning on dynamic graphs via self-attention networks.
\newblock In {\em 13th {ACM} Conference on Web Search and Data Mining, {WSDM}}, pages 519--527. {ACM}, 2020.

\bibitem{SeoDVB18}
Y.~Seo, M.~Defferrard, P.~Vandergheynst, and X.~Bresson.
\newblock Structured sequence modeling with graph convolutional recurrent networks.
\newblock In {\em 25th Neural Information Processing Conference, {ICONIP}}, volume 11301 of {\em Lecture Notes in Computer Science}, pages 362--373. Springer, 2018.

\bibitem{DBLP:journals/access/SkardingGM21}
J.~Skarding, B.~Gabrys, and K.~Musial.
\newblock Foundations and modeling of dynamic networks using dynamic graph neural networks: {A} survey.
\newblock {\em {IEEE} Access}, 9:79143--79168, 2021.

\bibitem{SouzaMKG22}
A.~H. Souza, D.~Mesquita, S.~Kaski, and V.~K. Garg.
\newblock Provably expressive temporal graph networks.
\newblock In {\em 35th Conference on Neural Information Processing Systems, {NeurIPS}}, 2022.

\bibitem{DBLP:conf/kr/CucalaGMK23}
D.~{Tena Cucala}, B.~{Cuenca Grau}, B.~Motik, and E.~V. Kostylev.
\newblock On the correspondence between monotonic max-sum {GNN}s and datalog.
\newblock In P.~Marquis, T.~C. Son, and G.~Kern{-}Isberner, editors, {\em 20th Conference on Principles of Knowledge Representation and Reasoning, {KR}}, pages 658--667, 2023.

\bibitem{DBLP:conf/kr/CucalaG24}
D.~J. {Tena Cucala} and B.~{Cuenca Grau}.
\newblock Bridging max graph neural networks and datalog with negation.
\newblock In P.~Marquis, M.~Ortiz, and M.~Pagnucco, editors, {\em 21st Conference on Principles of Knowledge Representation and Reasoning, {KR}}, 2024.

\bibitem{WR24}
P.~A. Wałęga and M.~Rawson.
\newblock Expressive power of temporal message passing.
\newblock {\em 39th Conference on Artificial Intelligence, {AAAI}}, 39(20):21000--21008, 2025.

\bibitem{XuRKKA20}
D.~Xu, C.~Ruan, E.~K{\"{o}}rpeoglu, S.~Kumar, and K.~Achan.
\newblock Inductive representation learning on temporal graphs.
\newblock In {\em 8th Conference on Learning Representations, {ICLR} 2020, Addis Ababa, Ethiopia, April 26-30, 2020}. OpenReview.net, 2020.

\bibitem{DBLP:conf/iclr/XuHLJ19}
K.~Xu, W.~Hu, J.~Leskovec, and S.~Jegelka.
\newblock How powerful are graph neural networks?
\newblock In {\em 7th Conference on Learning Representations, {ICLR}}, 2019.

\bibitem{DBLP:journals/corr/abs-2404-04393}
A.~Yang and D.~Chiang.
\newblock Counting like transformers: Compiling temporal counting logic into softmax transformers.
\newblock {\em CoRR}, abs/2404.04393, 2024.

\bibitem{YouDL22}
J.~You, T.~Du, and J.~Leskovec.
\newblock {ROLAND:} graph learning framework for dynamic graphs.
\newblock In {\em {KDD} '22: The 28th {ACM} {SIGKDD} Conference on Knowledge Discovery and Data Mining}, pages 2358--2366. {ACM}, 2022.

\bibitem{YuYZ18}
B.~Yu, H.~Yin, and Z.~Zhu.
\newblock Spatio-temporal graph convolutional networks: {A} deep learning framework for traffic forecasting.
\newblock In {\em 27th Joint Conference on Artificial Intelligence, {IJCAI}}, pages 3634--3640. ijcai.org, 2018.

\end{thebibliography}


\appendix


\section{Omitted definitions}
\label{app:definitions}

In this section, we provide full formal definitions of concepts that were introduced informally
in the main part of the paper. 

\subsection{Feedforward neural networks}
\label{def:fnn}

We present a formal definition of the classical feedforward neural network (FNN) model.
\begin{definition}
An \emph{FNN node} $v$ is a function from $\mathbb{R}^m$ to $\mathbb{R}$, where $m \in \mathbb{N}$ is the
\emph{input dimension}, computing
\begin{displaymath}
    v(x_1, \dotsc, x_n) = \sigma(b + \sum_{i=1}^n w_i x_i),
\end{displaymath} 
where $w_i \in \mathbb{Q}$ are the \emph{weights}, $b \in \mathbb{Q}$ is the \emph{bias},
and $\sigma: \R \to \R$ is the activation function.
An \emph{FNN layer} $\ell$ is a tuple consisting of some number $n \in \mathbb{N}$ FNN nodes
$(v_1, \dotsc, v_n)$, all having the same input dimensionality $m$. Then, $\ell$ computes the function
from $\mathbb{R}^m$ to $\mathbb{R}^n$, where $n$ is the \emph{output dimension}, given by
\begin{displaymath}
    \ell(x_1, \dotsc, x_m) = (v_1(x_1, \dotsc, x_m), \dotsc, v_n(x_1, \dotsc, x_m)).
\end{displaymath}
Finally, a \emph{feedforward neural network (FNN)} $N$ is a tuple of some $k \in \N$ FNN layers
$(\ell_1, \dotsc, \ell_k)$, where the input dimension of $\ell_{i+1}$ is equal to
the output dimension of $\ell_i$ for all $i < k$. Then, $N$ computes the function
from $\mathbb{R}^{m_1}$ to $\mathbb{R}^{n_k}$, where $m_1$ is the input dimensionality of nodes in layer $l_1$ and 
$n_k$ is the number of nodes in layer $l_k$, given by
\begin{displaymath}
    N(x_1, \dotsc, x_{m_1}) = \ell_k(\dotsb \ell_1(x_1, \dotsc, x_{m_1}) \dotsb).
\end{displaymath}
\end{definition}

Given this definition, we denote by $\FNN{\trelu}$ the class of all FNN where all nodes exclusively 
use $\trelu(x) = \max(0,\min(x,1))$ as the activation function. 

\subsection{The product logic $\PTLK$}
\label{def:ptlk}

In the following, we define one of the key logics we consider, namely $\PTLK$.
\begin{definition}
    \label{sec:logic2tgnn;def:ptlk}
    We define the logic \PTLK{} given by all formulae $\varphi$ defined by the grammar:
    \begin{displaymath}
        \varphi ::= c_j \mid \neg \varphi \mid \varphi \land \varphi \mid \Diamond \varphi \mid \prev{} \varphi \mid 
    \past{} \varphi      
    \end{displaymath}
    where $0 \le j < k$ for some set of $k$ colours.
    Let $(\TG,v)$ be a coloured, pointed temporal graph. Since the temporal operators in this logic are oblivious of the exact timestamps 
    in $\TG$ we assume that $\TG$ is discrete. We say that $(\TG,v)$ satisfies formula $\varphi$, written $(\TG,(v,t_n)) \models \varphi$, if given:
    \begin{align*}
        &\TG,(v,t_i) \models c_j && \text{iff} && (c_{t_i}(v))_j = 1,  \\
        &\TG,(v,t_i) \models \neg \varphi && \text{iff} && \TG, (v,t_i) \not\models \varphi, \\ 
        &\TG,(v,t_i) \models \varphi_1 \land \varphi_2 && \text{iff} && \TG,(v,t_i) \models \varphi_1 \text{ and } \TG, (v,t_i) \models \varphi_2, \\ 
        &\TG,(v,t_i) \models \Diamond \varphi && \text{iff} && \text{there is } u \in V \text{ s.t. } \{v,u\} \in E_{i} \text{ and } \TG,(u,t_i) \models \varphi, \\
        &\TG,(v,t_i) \models \prev{} \varphi && \text{iff} && t\neq 0 \text{ and } \TG,(v,t_{i}-1) \models \varphi, \\ 
        &\TG,(v,t_i) \models \past{} \varphi && \text{iff} && \text{there is } t_j < t_i \text{ such that } \TG,(v,t_j) \models \varphi.   
    \end{align*}
\end{definition}

\subsection{The product logics $\tlogbenedikt$ and $\tlognunn$.}
\label{def:benediktnunn}

We give formal definitions of the logics $\tlogbenedikt$ and 
$\tlognunn$ based on the modal logics presented in \cite{BenediktLMT24} and \cite{NunnSST24}.
These follow the same line as the definition of $\PTLK$ given in Definition~\ref{sec:logic2tgnn;def:ptlk}.
\begin{definition}
    \label{app:logic2tgnn;def:tlogbenedikt}
    Formulae of $\tlogbenedikt$ are defined by the grammar:
    \begin{displaymath}
        \varphi ::=  \top \mid c_i \mid \sum_{i=1}^{k} a_i \cdot \# \varphi  \leq b \mid \neg \varphi \mid \varphi \land \varphi \mid \prev{} \varphi \mid 
    \past{} \varphi      
    \end{displaymath}
    where $c_i$ ranges over propositional variables in a finite set $C$ of colours, and 
    $a_i, b \in \Q$.
    Let $(\TG,v)$ be a coloured, pointed temporal graph. Since the temporal operators in this logic are oblivious of the exact timestamps 
    in $\TG$ we assume that $\TG$ is discrete. We say that $(\TG,v)$ satisfies formula $\varphi$, written $(\TG,(v,t_n)) \models \varphi$, if given:
    \begin{align*}
        &\TG,(v,t_i) \models \top && \text{iff} && \text{true},  \\
        &\TG,(v,t_i) \models c_j && \text{iff} && (c_{t_i}(v))_j = 1,  \\
        &\TG,(v,t_i) \models \sum_{i=1}^{k} a_i \cdot \# \varphi_i  \leq b && \text{iff} && \sum_{i=1}^{k} a_i \cdot |\{u \in V \mid \{v,u\} \in E_t \text{ and } \TG,(v,t_i) \models \varphi_i\}|  \leq b, \\
        &\TG,(v,t_i) \models \neg \varphi && \text{iff} && \TG, (v,t_i) \not\models \varphi, \\ 
        &\TG,(v,t_i) \models \varphi_1 \land \varphi_2 && \text{iff} && \TG,(v,t_i) \models \varphi_1 \text{ and } \TG, (v,t_i) \models \varphi_2, \\ 
        &\TG,(v,t_i) \models \prev{} \varphi && \text{iff} && t_i \neq 0 \text{ and } \TG,(v,t_i-1) \models \varphi, \text{ and} \\ 
        &\TG,(v,t_i) \models \past{} \varphi && \text{iff} && \text{there is } t' < t_i \text{ such that } \TG,(v,t') \models \varphi.   
    \end{align*} 
\end{definition}
Although it does not affect any of our results, we note that \cite{BenediktLMT24} consider MPNNs operating
on directed graphs, whereas our focus is on static graphs with undirected edges. Consequently, there is a technical difference in the
definition of the semantics of the $\sum_{i=1}^{k} a_i \cdot \# \varphi_i \leq b$ quantifier compared to \cite{BenediktLMT24}.
\begin{definition}
    \label{app:logic2tgnn;def:tlognunn}
    Formulae of $\tlognunn$ are defined by the grammar:
    \begin{displaymath}
        \varphi ::=  c \mid \sum_{i=1}^{k} a_i \cdot 1_{\varphi}+ \sum_{i=1}^{k'} b_{i} \cdot \# \varphi \leq d \mid \neg \varphi \mid \varphi \land \varphi \mid \prev{} \varphi \mid 
    \past{} \varphi      
    \end{displaymath}
    where $c$ ranges over propositional variables in a finite set $C$ of colours, and 
    $a_i, b_i, d \in \mathbb{Z}$.

    Let $(\TG,v)$ be a coloured, pointed temporal graph. Since the temporal operators in this logic are oblivious of the exact timestamps 
    in $\TG$ we assume that $\TG$ is discrete. We say that $(\TG,v)$ satisfies formula $\varphi$, written $(\TG,(v,t_n)) \models \varphi$, if given:
    \begin{align*}
        &\TG,(v,t_i) \models c_j && \text{iff} && (c_{t_i}(v))_j = 1,  \\
        &\TG,(v,t_i) \models \neg \varphi && \text{iff} && \TG, (v,t_i) \not\models \varphi, \\ 
        &\TG,(v,t_i) \models \varphi_1 \land \varphi_2 && \text{iff} && \TG,(v,t_i) \models \varphi_1 \text{ and } \TG, (v,t_i) \models \varphi_2, \\ 
        &\TG,(v,t_i) \models \prev{} \varphi && \text{iff} && t_i\neq 0 \text{ and } \TG,(v,t_i-1) \models \varphi, \\ 
        &\TG,(v,t_i) \models \past{} \varphi && \text{iff} && \text{there is } t' < t_i \text{ such that } \TG,(v,t') \models \varphi, \text{ and} \\
        & \TG,(v,t_i) \models \sum_{i=1}^{k} a_i \cdot 1_{\varphi_i}+ \sum_{i=1}^{k'} b_{i} \cdot \# \varphi'_{i} \leq d && \text{iff} &&  
    \end{align*} 
    \begin{displaymath}
        \sum_{i=1}^{k} a_i \cdot \begin{cases}
            1 & \text{if } \TG,(v,t_i) \models \varphi_i \\ 
            0 & \text{otherwise,}
        \end{cases} + \sum_{i=1}^{k'} b_i \cdot |\{u \in V \mid \{v,u\} \in E_t \text{ and } \TG,(v,t_i) \models \varphi'_i\}|  \leq d, 
    \end{displaymath}
\end{definition}
We note here that we utilised a normal form of $\lognunn$ (refer to Theorem 1 in \cite{NunnSST24})
to simplify the definition of $\tlognunn$. As the authors point out,
it is straightforward to observe that each formula of the original $\lognunn$ 
syntax can be efficiently transformed into this normal form.

\subsection*{The time function time2vec}
\label{def:time2vec}

In the following, we define time2vec functions \cite{Kazemi2019}.
\begin{definition}
We define $\phi(t_i-t_l) = \mathsf{t2v}(t_i-t_l)$ as
\begin{align*}
    \mathsf{t2v}(t)_j =
    \begin{cases}
        w_j t + b_j & \text{if } j=0, \\
        \sigma(w_j t + b_j) & \text{otherwise,}
    \end{cases}
\end{align*}
where $w_j, b_j \in \Q$ and $\sigma$ is some periodic activation function.
\end{definition}
Correspondingly, we denote the class of all such functions by $\mathcal{Q}_{\mathsf{time2vec}}$.

\section{Omitted proofs}
\label{app:logic2tgnn}

In this section, we provide all the formal proofs for the results presented in
this work. This includes comprehensive proofs for the results
of the paper that were merely outlined in the main section.

\subsection{Proof of Theorem~\ref{sec:logic2tgnn;thm:ptlk}}
\label{proof:ptlk}
\ptlk*
\begin{proof}
    Let $\varphi$ be a formula of $\PTLK$ as defined in Definition~\ref{sec:logic2tgnn;def:leq} 
    with $m$ atomic subformulas. Let
    \begin{displaymath}
           \varphi_1, \dotsc, \varphi_{m}, \varphi_{m+1}, \dotsc, \varphi_n
    \end{displaymath}
    be an enumeration of the subformulas of $\varphi$ such that all atomic
    formulas are the $\varphi_1, \dotsc, \varphi_m$, and $\varphi_i \in \mathit{sub}(\varphi_j)$
    implies $i \le j$. In particular, we have $\varphi_n = \varphi$.

    We begin by describing how the TGNN $\model_\varphi$ is constructed. 
    We have $\model_\varphi = (M, \out)$, where the MPNN $M = (l^1, \dots, l^{n-m+1})$ and the layer 
    $l^{(i)}$ with $i \leq n-m$ is given by $(\com^{(i)}, \sum)$ with $\com^{(i)}(\boldsymbol{x}, \boldsymbol{y}) = \trelu(C\boldsymbol{x} + A\boldsymbol{y} + b)$, 
    and where $\sum$ denotes entrywise sum as aggregation. 
    The exact form of the $3n \times 3n$ matrices $C$, $A$, and the $n$-dimensional vector $b$ depends on $\varphi_{m+i}$ as follows:
    \begin{itemize}
    \item if $\varphi_{m+i} = \neg \varphi_j$, we have $C_{m+i,j} = -1$, and $b_{m+i} = 1$,
    \item if $\varphi_{m+i} = \varphi_{j_1} \land \varphi_{j_2}$, we have $C_{m+i,j_1} = C_{m+i,j_2} = 1$, and $b_{m+i} = -1$,
    \item if $\varphi_{m+i} = \Diamond \varphi_j$, we have $A_{m+i,j} = 1$,
    \item if $\varphi_{m+i} = \prev{} \varphi_j$, we have $C_{m+i, n+j} = 1$,
    \item if $\varphi_{m+i} = \past{} \varphi_j$, we have $C_{m+i,2n+j} = 1$, and 
    \end{itemize}
    for all $j \leq m$ we have $C_{j,j} = 1$. All other entries of $C$, $A$, and $b$ are zero. 
    This implies that all layers $l_i$ with $i \leq n-m$ use the same parameters. 
    The layer $l^{n-m+1} = (\com^{n-m+1},\sum)$ is given by  
    $\com^{n-m+1}(\boldsymbol{x}, \boldsymbol{y}) = \trelu(C'\boldsymbol{x} + A'\boldsymbol{y} + \boldsymbol{0})$, 
    where $A'$ is the $2n \times 3n$ all-zero matrix, $\boldsymbol{0}$ is the $2n$-dimensional all-zero vector, 
    and $C'$ is the $2n \times 3n$ matrix with $C'_{j,j} = 1$, $C'_{n+j,n+j} = 1$, and $C'_{n+j,2n+j} = 1$ for all $j \leq n$. 
    All other entries are zero.
    The output function is given by $\out(x_1, \dotsc, x_{2n}) = \trelu(x_n)$. 
    It is straightforward to see that $\model_\varphi \in \idTGNN[\mathcal{M}]$.

    Let $(\TG,v)$ be a pointed temporal graph where $\TG$ is of length $k$ and $V$ is its set of nodes. 
    Regarding correctness, we prove the following statement: for all nodes $u \in V$, all timepoints 
    $t_i$ with $i \leq k$ of $\TG$ and subformulas $\varphi_j$ with $j \leq n$ of $\varphi$ we have that
    \begin{enumerate}
        \item[a)] $\TG, (u,t_i) \models \varphi_j$ if and only if $\h^{(\min(0,j-m))}_u(t_i)_j = 1$, 
        \item[b)] $\TG, (u,t_{i}-1) \models \varphi_j$ if and only if $\h^{(\min(0,j-m))}_u(t_i)_{n+j} = 1$, and
        \item[c)] $\TG, (u,t_{i'}) \models \varphi_j$ for some $i' < i$ if and only if $\h^{(\min(0,j-m))}_u(t_i)_{2n+j} = 1$.
    \end{enumerate}
    We prove this statement via strong induction on $i$ and $j$. 
    
    \paragraph{Case: timestamp $t_1$.} First, let $i = 1$ and $j \in \{1, \dotsc, m\}$, and fix some $u \in V$.
    The assumption $j \in \{1, \dotsc, m\}$ implies that $\varphi_j$ is an atomic formula. Then, statement (a) is directly implied 
    by the form of $h^{(0)}_u(t_1)$, including the colours of node $u$. Similarly, we have that statements (b) and (c) are given as $\model_\varphi$ initialises the dimensions $n$
    to $3n$ of $h^{(0)}_u(t_1)$ with $0$, which is correct as we are considering the first timestamp of \TG. 
    Next, assume that the statement holds for $i = 1$, $j$, and all $u \in V$. Consider the case of $j+1$, fix some $u \in V$, and
    focus on statement (a). The Boolean cases $\varphi_{j+1} = \neg \varphi_{l_1}$ and $\varphi_{j+1} = \varphi_{l_1} \land \varphi_{l_2}$ are a straightforward implication
    of the form of matrix $C$, vector $b$, as well as the activation function $\trelu$, and
    the fact that statement (a) holds for all $l_1, l_2 \leq j$. Similarly, the case $\Diamond \varphi_{l_1}$ is implied
    by the form of matrix $A$, the activation function $\trelu$, and that statement (a) holds for all $l_1 \leq j$ and all $w \in V$,
    including the neighbours of $u$. In the case of $\varphi_{j+1} = \prev{} \varphi_{l_1}$ or $\varphi_{j+1} = \past{} \varphi_{l_1}$, we rely on the
    fact that $i = 1$, which means that these are necessarily false. By the hypothesis, this is implied by (b) and (c) for all $l_1 \leq j$ as 
    matrix $C$ utilises the corresponding dimensions $n+l_1$ and $2n+l_1$, respectively, which are all $0$. The fact that we consider $i=1$ and that 
    all dimenstions $l > n$ are 0 also immediately implies statements (b) and (c).

    \paragraph{Case: timestamp $t_i$ with $i > 1$.} Next, assume that statements (a) to (c) hold for $i$, all $j \leq n$, and all $u \in V$, and consider the case $i+1$.
    Once again, the argument for statement (a) with $j \in \{1, \dotsc, m\}$ is directly implied by the fact that $\model_\varphi$ stores
    the colours of each node $u$ in the respective dimension. Before addressing (b) and (c) for these $j$, consider the following two observations:
    \begin{enumerate}
        \item TGNN $\model_\varphi$ ensures for all timestamps $t_i$, $j \leq n$, and $u \in V$  that
        $\h^{(0)}_u(t_{i+1})_{n+j} = \h^{(n-m)}_u(t_{i})_{j}$, and it ensures that $\h^{(0)}_u(t_{i+1})_{2n+j} = 1$ if and only if 
        $\h^{(n-m)}_u(t_{i'})_{j}$ for some $i' \leq i$.
        \item TGNN $\model_\varphi$ ensures for all timestamps $t_i$, $u \in V$ that $\h^{(0)}_u(t_{i})_j = \h^{l}_u(t_i)_j$ holds
        for all $j \in \{n+1, \dotsc, 3n\}$ and $l \leq n-m$, and it ensures that $\h^{(j)}_u(t_{i})_j = \h^{l}_u(t_i)_j$ holds for all $j \leq n$ and $j \leq l \leq n-m$.
    \end{enumerate} 
    Informally, the first observation implies that dimensions $n+1$ to $2n$ store the previous state of a node, and $2n+1$ to
    $3n$ store the disjunction over all previous states of a node, where $1$ is interpreted as true and $0$ as false. This is
    achieved by the way we built layer $l^{n-m+1}$.
    The second observation simply states that, within a single timepoint $t_i$, the TGNN $\model_\varphi$ does not alter dimensions $n+1$ to $3n$
    within layers $l^{1}$ to $l^{n-m}$ and that if the semantics of $\varphi_j$ are computed in the $j$-th layer, they are preserved in subsequent layers.
    Now, consider statements (b) and (c) for $i+1$, $j \in \{1, \dotsc, m \}$ and some fixed $u \in V$. Here, the first observation 
    and induction hypothesis directly imply these statements. Therefore, keep $u \in V$ fixed and consider $j+1$, while assuming that (a) to (c) hold
    for all $j' \leq j$. The Boolean and modal cases are argued exactly as before. Thus, consider $\varphi_{j+1} = \prev{} \varphi_l$ for 
    some $l \leq j$. Using the first and second observations, we know that the semantics of $\varphi_l$ at timepoint $t_i$ are
    stored in $\h^{(j+1)}_u(t_{i+1})_{n+l}$, which is utilised by $\model_\varphi$ to compute the semantics of $\prev{} \varphi_l$. Given
    the induction hypothesis, this is correct. Similarly, the case $\varphi_{j+1} = \past{} \varphi_l$ is implied by the two observations
    and the induction hypothesis, which state that $\h^{(j+1)}_u(t_{i+1})_{2n+l}$ stores a $1$ if $\varphi_l$ was true at any timepoint before $t_{i+1}$ and $0$ otherwise.
    Given this, the correctness is immediate.

    Finally, the correctness of the theorem is given by statement (a) for $t_k$ and $\varphi_n = \varphi$ in combination 
    with observation that $\out$ uses the $n$-th dimension of $\h^{n-m+1}_v(t_k)$ to compute the overall output.
\end{proof}

\subsection{Proof of Theorem~\ref{sec:log2tgnn;thm:benediktnunn}}
\label{proof:benediktnunn}
For a formal definition of the logics $\tlogbenedikt$ and $\tlognunn$, we refer to the corresponding 
subsection of Appendix~\ref{def:benediktnunn}.

The key results we rely on are Theorem~26 of \cite{BenediktLMT24} and Theorem~1 of \cite{NunnSST24}. 
However, we need a stronger form of these results that includes the capturing of the semantics of each 
subformula with a MPNN. 
\begin{definition}
    Let $\varphi$ be an inductively built formula interpreted over pointed static graphs, 
    and let $M$ be an MPNN. 
    We say that \emph{$M$ inductively captures $\varphi$} if there exists an enumeration 
    of all subformulas $\varphi_1, \dotsc, \varphi_m$ of $\varphi$, where 
    $\varphi_i \in \mathit{sub}(\varphi_j)$ implies $i \leq j$, such that for all pointed graphs $(G,v)$
    and formulas $\varphi_i$ there is layer $j_i$ of $M$ such that
    \begin{itemize}
        \item if $(G,v) \models \varphi_i$, then $\h^l(v)_i = 1$ for all $l \geq j_i$, and 
        \item if $(G,v) \not\models \varphi_i$, then $\h^l(v)_i = 0$ for all $l \geq j_i$,
    \end{itemize}
    where $\h^l(v)$ represents the state of $v$ computed by the $j_i$-th layer of $M$. Furthermore, 
    we require that $j_{i-1} \leq j_{i}$ for all $i \in \{1, \dotsc, m\}$.
    We extend the notion of inductive capture to a logic $\mathcal{L}$ over pointed static graphs 
    and a class of MPNNs $\mathcal{M}$ in the obvious way and denote it by $\mathcal{L} \trianglelefteq \mathcal{M}$.
\end{definition}

Given this understanding, a close examination of the arguments employed in \cite{BenediktLMT24}
and \cite{NunnSST24} directly implies the inductive capture of the logics $\logbenedikt$ and $\lognunn$ 
by the respective classes of MPNNs.
\begin{lemma}
    \label{app:logic2tgnn;lem:benediktnunn_inductive}
    Let $\logbenedikt$ and $\lognunn$ be the logics introduced in \cite{BenediktLMT24} and 
    \cite{NunnSST24} and $\mathcal{O}\mathcal{L}\trelu$-MPNN and $\mathcal{M}_{\lognunn}$ 
    the respective classes of MPNNs. We have that 
    $\logbenedikt \trianglelefteq \mathcal{O}\mathcal{L}\trelu\text{-MPNN}$ and 
    $\lognunn \trianglelefteq \mathcal{M}_{\lognunn}$.
\end{lemma}
\begin{proof}
    Both results, namely Theorem 26 from \cite{BenediktLMT24} and Theorem 1 from \cite{NunnSST24},
    are constructive in nature: for each formula $\varphi$ the authors construct an MPNN $M_\varphi$
    that captures its semantics. Moreover, these constructions are inductive, meaning that all subformulas of $\varphi$ are
    enumerated and the MPNN $M_\varphi$ is built such that it captures the semantics of the $i$-th subformula with its 
    $i$-th layer and preserves it in subsequent layers. This establishes that these results in fact imply inductive capturing.
\end{proof}

Now, we are set to prove that both $\tlogbenedikt$ and $\tlognunn$ are captured by the respective classes of TGNNs.
\benediktnunn*
\begin{proof}
    First, we address the statement that $\tlogbenedikt \leq \idTGNN[\mathcal{O}\mathcal{L}\trelu\text{-GNN}]$. 
    Let $\varphi \in \tlogbenedikt$ with $k$ different subformulas of the form $\prev{} \psi$ or 
    $\past{} \psi$. 
    Let
    \begin{displaymath}
           \varphi_1, \dotsc, \varphi_{m_1}, \varphi_{m_1+1}, \dotsc, \varphi_{m_2}, \varphi_{m_2+1}, \dotsc, 
           \varphi_{m_k}, \varphi_{m_k+1}, \dotsc, \varphi_n
    \end{displaymath}
    be an enumeration of the subformulas of $\varphi$ such that 
    $\varphi_i \in \mathit{sub}(\varphi_j)$ implies $i \le j$, 
    $\varphi_i \neq \prev{} \psi$ and $\varphi_i \neq \past{} \psi$ if $i \notin \{m_1, \dotsc, m_k\}$, 
    and we have $\varphi_n = \varphi$.  
    Furthermore, we assume for each set $S_i = \{\varphi_{m_i+1}, \dotsc, \varphi_{m_{i+1}-1}\}$ with $m_0 = 1$ 
    that the enumeration is such that Lemma~\ref{app:logic2tgnn;lem:benediktnunn_inductive} applies to each 
    formula of $S_i$, where we interpret potential occurrence of $\varphi_j$ with $j \leq m_i$ as some 
    fresh atomic formula. 

    Let $\model_\varphi = (M, \out)$, where $\out(x_1, \dotsc, x_{2n}) = \trelu(x_n)$ and $M$ is constructed as follows.
    The initial layers are constructed to capture the semantics of the formulas $\varphi_1$
    to $\varphi_{m_1-1}$ in dimensions $1$ to $m_1-1$. The existence of such message-passing layers is guaranteed by the existence
    of MPNNs capturing the formulas of $S_0 = \{\varphi_1, \dotsc, \varphi_{m_1-1}\}$, as indicated by Lemma~\ref{app:logic2tgnn;lem:benediktnunn_inductive}.
    Furthermore, these message-passing layers are built with input and output dimensionality of $3n$, where dimensions
    $m_1$ to $3n$ are mapped by the identity function within the range $[0,1]$. The argument that these can be interconnected in $M$
    to form a well-formed MPNNs is given by the fact that the class
    $\mathcal{O}\mathcal{L}\trelu$-GNN encompasses allows for arbitrary single-layer FNN with truncated-ReLU activations as combinations.
    Next, consider $\varphi_{m_1}$ which is either $\prev{}\varphi_i$ or 
    $\past{}\varphi_i$ for some $i \leq m_1$. In the case of $\prev{}\varphi_i$, we add a layer that maps dimension
    $n+i$ to dimension $m_1$, and in the case of $\past{}\varphi_i$, we add a layer that maps dimension $2n+i$ to
    dimension $m_1$. Again, this is feasible due to the fact that $\mathcal{O}\mathcal{L}\trelu$-MPNN includes arbitrary
    single-layer FNN with truncated-ReLU activations as combinations. Subformulas $\varphi_{m_1+1}$ to $\varphi_{m_2-1}$
    are handled like $\varphi_1$ to $\varphi_{m_1-1}$ using the stack of MPNNs implied by Lemma~\ref{app:logic2tgnn;lem:benediktnunn_inductive} 
    for $S_1 = \{\varphi_{m_1+1}, \dotsc, \varphi_{m_2-1}\}$, combined accordingly. Then, $\varphi_{m_2}$ is addressed like $\varphi_{m_1}$ and so forth. Finally, we add
    a layer $l$ that computes the exact same function as $l^{n-m+1}$ in the proof of Theorem~\ref{sec:logic2tgnn;thm:ptlk}, which ensures 
    that previous and past semantics are preserved.
    Given this construction, we have that $\model_\varphi \in \idTGNN[\mathcal{O}\mathcal{L}\trelu\text{-MPNN}]$.

    The correctness argument follows the exact same line of reasoning as that in Theorem~\ref{sec:logic2tgnn;thm:ptlk},
    namely, arguing inductively over the timepoints $i$ of a temporal graph and subformulae $\varphi_j$.
    However, because we have not explicitly constructed $\model_\varphi$, we utilize the following arguments.
    For non-temporal subformulae $\varphi_j$, we rely on the fact that Lemma~\ref{app:logic2tgnn;lem:benediktnunn_inductive}
    applies to $\mathcal{O}\mathcal{L}\trelu$-MPNN, which means that $\varphi_j$ is captured after the respective layer.
    Note that this includes the observation that the assumptions we made for $S_i$ in order to apply Lemma~\ref{app:logic2tgnn;lem:benediktnunn_inductive} 
    makes no difference in the overall computation of $M$. Otherwise, we use precisely the same arguments. 

    Finally, for the case of $\tlognunn \leq \idTGNN[\tlognunn]$, we note that Lemma~\ref{app:logic2tgnn;lem:benediktnunn_inductive}
    is applicable to both $\lognunn$ and $\mathcal{M}_{\lognunn}$. Moreover, the properties regarding the combination of MPNN layers we needed above
    are also given for $\mathcal{M}_{\lognunn}$. Therefore, the argumentation presented above is equivalently applicable.
\end{proof}

\subsection{Proof of Theorem~\ref{th:TandGweak}}
\label{proof:tandgweak}
\tandgweak*
\begin{proof}
    Let $\varphi \in \PTLK$ be the formula $\varphi = \Diamond(c_1 \land \prev{}\Diamond c_2)$. We prove that there is no 
    TGNN $T \in \tandgTGNN[\mathcal{M}, \mathcal{C}]$ for any class of MPNNs $\mathcal{M}$ and 
    functions $\mathcal{C}$ such that for all pointed temporal graphs $(\TG,v)$ we have that  $(\TG,v) \models \varphi$ 
    if and only if $T(\TG,v) = 1$.

    We consider the two pointed temporal graphs $(\TG,v)$ and $(\TG',v')$ 
    fully specified in form of Figure~\ref{fig:tandgweak_ex}.
    It can be easily verified that $(\TG,v) \models \varphi$ and $(\TG',v') \not\models \varphi$.
    Let $T \in \tandgTGNN[\hat{\mathcal{M}}, \mathcal{C}]$ be some TGNN. 
    It is straightforward to see that for $M_1$ of $T$, we have that (a) $M_1(G_1,v) = M_1(G'_1,v')$, $M_1(G_1,u_1) = M_1(G_1',w_1')$, $M_1(G_1,u_2) = M_1(G_1',w_2')$,  $M_1(G_1, w_1) = M_1(G_1',u_1')$, and $M_1(G_1,w_2) = M_1(G_1',u_2')$. 
    This follows from the fact that $M_1$ is a function and the input of 
    the aggregation is mutliset without any ordering.
    Similarly, this implies (b) $M_1(G_2,x) = M_1(G'_2,x')$ for $x \in \{v,u_1,u_2,w_1,w_2\}$. Also, we have that $M_2$ outputs the same vector for 
    each pair of nodes $x$ and $x'$ at timestamp $t_1$, due to the fact that all labels are assumed to be $\boldsymbol{0}$ and $E_1 = E_1'$. 
    Thus, the equalities of (a) are preserved by the application of $\cell$ in the first timestamp $t_1$.
    
    Now, using (a), the same kind of argument gives that $M_2$ outputs the same vector for $v$ and $v'$ at timestamp $t_2$.
    Then, combining this with (b), we find that the function $\cell$ receives identical inputs in the cases of
    $(\TG,v)$ and $(\TG',v')$ at timestamp $t_2$, leading to the conclusion that either both are accepted or both are rejected by $T$. 
    Thus, $T$ does not accept the exact same set of pointed temporal graphs as $\varphi$.
\end{proof}

\subsection{Proof of Theorem~\ref{sec:variants_logic2tnn;thm:tandg}}
\label{proof:tandg}

\tandg*
\begin{proof}
    The proof follows a similar line of reasoning as our previous results, notably the arguments presented
    in the proof of Theorem~\ref{sec:logic2tgnn;thm:ptlk}. Let $\varphi$ be a formula of the fragment 
    $\mathcal{L}_1$ , and let $\varphi_1, \dotsc, \varphi_n $ where $\varphi_i \in \mathit{sub}(\varphi_j)$
    implies $i \le j$. Specifically, we have $\varphi_n = \varphi$.

    Let $\model_\varphi = (M_1, M_2, \cell, \out)$, where $\out(x_1, \dots, x_{2n}) = \trelu(x_n)$. The components
    $M_1$, $M_2$, and $\cell$ are constructed as outlined below. $M_1$ and $M_2$ consist of layers
    $l^{(i)}_1$ and $l^{(i)}_2$, respectively. In the case of $M_1$ we include a layer $l^{(0)}_1$ that 
    maps vector $\boldsymbol{x} \in \{0,1\}^m$ to $\boldsymbol{x} \con \boldsymbol{0} \in \{0,1\}^m \times \{0\}^{n-m}$ and 
    in the case of $M_2$ we include a layer $l^{(0)}_2$ that 
    maps vector $\boldsymbol{x} \in \{0,1\}^{2n}$ to $\boldsymbol{0} \con \boldsymbol{x} \in \{0\}^{n} \times \{0,1\}^{2n}$.
    Each layer $l^{(i)}_j$ with $j \in \{1,2\}$ is represented by $(\com^{(i)}_j, \sum)$, where $\sum$ means entrywise sum as aggregation and
    $\com^{(i)}_j(\boldsymbol{x}, \boldsymbol{y}) = \trelu(C_j\boldsymbol{x} + A_j\boldsymbol{y} + \boldsymbol{b}_j)$, where $C_1, A_1 \in \{0,1\}^{n\times n}$, $\boldsymbol{b}_1 \in \{0,1\}^n$ and $C_2,A_2 \in \{0,1\}^{3n\times 3n}$, 
    $\boldsymbol{b}_2\in \{0,1\}^{3n}$. The function $\cell$ is represented by an FNN $N_\cell$ of input dimensionality $3n$ and output dimensionality $2n$. 
    The entries of $C_j$, $A_j$, $\boldsymbol{b}_j$ with $j \in \{1,2\}$, and the exact form of $N_\cell$  
    are determined by the enumeration of the subformulas $\varphi_i$. We distinguish three cases, depending on the nature of $\varphi_i$.

    Firstly, let $\varphi_i$ be such that there is no subformula $Q\psi \in \sub(\varphi_i)$ with $Q \in \{\prev{}, \past{}\}$.
    In this case, we set the parameters of $C_1$, $A_1$, and $\boldsymbol{b}_1$ based on $\varphi_i$ as shown in the proof of
    Theorem~\ref{sec:logic2tgnn;thm:ptlk}. Informally, this means the semantics of $\varphi_i$ are checked by $M_1$.
    
    
    Secondly, consider $\varphi_i$ where for every subformula $c \in \sub(\varphi_i)$, there exists
    $Q\psi \in \sub(\varphi_i)$ with $Q \in \{\prev{}, \past{}\}$ such that $c \in \sub(\psi)$.
    We configure the parameters $C_2$, $A_2$, and $\boldsymbol{b}_2$ based on $\varphi_i$
    following the procedure described in the proof of Theorem~\ref{sec:logic2tgnn;thm:ptlk}.
    However, for $\varphi_i = \prev{} \varphi_j$, we use input dimension $n+j$, and for
    $\varphi_i = \past{} \varphi_j$, we use $2n+j$. 
    We also add a final layer to $M_2$, which maps
    vectors $\boldsymbol{x}_1 \con \boldsymbol{x}_2 \con \boldsymbol{x}_3$, where
    $\boldsymbol{x}_j \in \{0,1\}^n$, to $\boldsymbol{x}_1 \con \boldsymbol{x}_3$.

    Thirdly, formulas $\varphi_i$ that do not belong to the first or second category are handled by $N_\cell$
    of input dimensionality $3n$ and output dimensionality $2n$.
    The definition of the fragment $\mathcal{L}_1$ ensures that
    $\varphi_i = \neg \psi$ or $\varphi_i = \psi_1 \land \psi_2$. We refer to works such as \cite{SL22}, which demonstrate
    how to construct single-layer FNNs to check Boolean conditions. Besides processing these $\varphi_i$, the FNN
    $N_\cell$ maps input $x_i$ corresponding to $\varphi_i$ of the first category 
    by using the identity (simply realised by $\trelu(x)$) to output $y_i$ and input $x_{n+i}$ corresponding to 
    $\varphi_i$ from the second category are mapped identically to output $y_i$ as well. Additionally, for each $j \leq n$, 
    we add a component that computes $\trelu(x_j + x_{n+j} + x_{2n+j})$ as the $n+j$th output, ensuring that all $\past{}\psi$
    subformulas are correctly processed.

    Regarding correctness, the argument follows the same inductive approach as in Theorem~\ref{sec:logic2tgnn;thm:ptlk}.
    To avoid repetition, we provide a high-level outline here. Concerning the MPNN $M_1$,
    it is immediately implied by the construction of Theorem~\ref{sec:logic2tgnn;thm:ptlk} that it functions as expected, 
    meaning that it correctly captures the semantics of $\varphi_i$ of the first category. This is due to the fact that it is only utilised to check
    formulas devoid of temporal operators. For $M_2$, consider that at the initial timestamp $t_1$,
    its inputs are pointed graphs $((V_1,E_1, [u \mapsto \boldsymbol{0}]), v)$, where all labels of
    $G_1$ are replaced by $\boldsymbol{0}$. Since $M_2$ is only used to verify formulas concerning past timepoints, 
    the base case is correct.
    Otherwise, the inputs for $M_2$ are $((V_i, E_i, [u\mapsto \h_u(t_{i-1})]), v)$, where $\h_u(t_{i-1})$ denotes the output
    of $\cell$ for node $u$ at the previous timestamp. Here, our construction ensures that $\h_u(t_{i-1}) = \h^{\prev{}}_u(t_{i-1}) \con \h^{\past{}}_u(t_{i-1})$,
    where $\h^{\prev{}}_u(t_{i-1}) \in \{0,1\}^n$ contains the semantics of each subformula $\varphi_i$ at timestamp $t_{i-1}$ and 
    $\h^{\past{}}_u(t_{i-1}) \in \{0,1\}^n$ contains semantics of each subformula $\varphi_i$, disjunctively combined over all
    $t_j$ with $j \leq i-1$. Due to the way $\mathcal{L}_1$ is defined, these are the necessary informations to compute the semantics of $\varphi_i$ of the second category 
    at timepoint $t_i$. Finally, $N_\cell$ uses the outputs of $M_1$ and $M_2$ to compute the semantics of the remaining subformulas. 
    Due to the way the fragment $\mathcal{L}_1$ is defined, these can not be of the form $\Diamond\psi$, which $N_\cell$ could not handle,
    or $Q\psi$ with $Q \in \{\prev{}, \past{}\}$, which are already handled by $M_2$. This leaves only Boolean formulas. Otherwise, $N_\cell$
    ensures consistency and, thus, produces the output $\h_u(t_i) \con \h_u^{\past{}}(t_i)$, where
    $\h_u(t_i) \in \{0,1\}^n$ contains the semantics of all subformulas of $\varphi$ at timepoint $t_i$ and 
    $\h^{\past{}}_u(t_{i}) \in \{0,1\}^n$ contains semantics of all subformulas, disjunctively combined over all
    $t_j$ with $j \leq i$. Finally, the output function $\out(x_1, \dotsc, x_{2n}) = \trelu(x_n)$ ensures that $T_\varphi$ outputs the semantics of 
    $\varphi_n = \varphi$.
\end{proof}

\subsection{Proof of Theorem~\ref{thm:globTGNNweak}}
\label{proof:globtgnnweak}
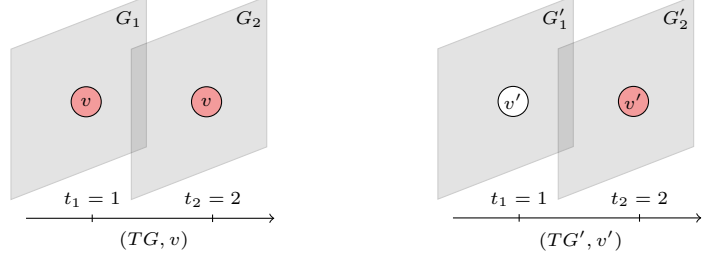
\begin{figure}[h]
    \centering
    \begin{minipage}{0.2\textwidth}
    \centering
            \begin{tikzpicture}[
dot/.style = {draw, circle, minimum size=0.4cm,
              inner sep=0pt, outer sep=0pt},
dot/.default = 6pt
]
\scriptsize
\pgfmathsetmacro{\tline}{-0.25}
\pgfmathsetmacro{\w}{1.8}
\pgfmathsetmacro{\h}{2.0}
\pgfmathsetmacro{\inh}{0.7}
\pgfmathsetmacro{\dist}{1.6}
\pgfmathsetmacro{\Ax}{0.5}
\pgfmathsetmacro{\Ay}{1.4}
\pgfmathsetmacro{\Bx}{0.6}
\pgfmathsetmacro{\By}{0.6}
\pgfmathsetmacro{\Cx}{1}
\pgfmathsetmacro{\Cy}{1.3}
\pgfmathsetmacro{\Dx}{1.1}
\pgfmathsetmacro{\Dy}{2}
\pgfmathsetmacro{\Ex}{1.3}
\pgfmathsetmacro{\Ey}{0.8}


\draw[->] (1.8,\tline) -- (5.1,\tline);
\node at (3.5,\tline-0.3) {$(\TG,v)$};

\foreach \x in {1,...,2}
{
\draw[fill=gray!90!black,opacity=0.2] (\x*\dist,0) -- (\x*\dist+\w,\inh) -- (\x*\dist+\w,\inh+\h) -- (\x*\dist,\h) -- cycle;

\node at (\x*\dist + 0.87 * \w, \inh+ \h - 0.3) {$G_{\x}$};
\node at (\x*\dist + 0.6 * \w, 0) {$t_{\x}=\x$};
\draw[-] (\x*\dist + 0.6 * \w, \tline+0.05) -- (\x*\dist + 0.6 * \w, \tline-0.05);

\node (A\x) at (\x*\dist+\Ax, 0+\Ay) {};
\node (B\x) at (\x*\dist+\Bx, 0+\By) {};
\node (C\x) at (\x*\dist+\Cx, 0+\Cy) {};
\node (D\x) at (\x*\dist+\Dx, 0+\Dy) {};
\node (E\x) at (\x*\dist+\Ex, 0+\Ey) {};
}


\node[dot=9pt,draw=black, fill=myred] at (C1) {$v$};
\node[dot=9pt,draw=black, fill=myred] at (C2) {$v$};
\end{tikzpicture}
    \end{minipage}%
    \hspace{.2\textwidth}
    \begin{minipage}{0.2\textwidth}
    \centering
            \begin{tikzpicture}[
dot/.style = {draw, circle, minimum size=0.4cm,
              inner sep=0pt, outer sep=0pt},
dot/.default = 6pt
]
\scriptsize
\pgfmathsetmacro{\tline}{-0.25}
\pgfmathsetmacro{\w}{1.8}
\pgfmathsetmacro{\h}{2.0}
\pgfmathsetmacro{\inh}{0.7}
\pgfmathsetmacro{\dist}{1.6}
\pgfmathsetmacro{\Ax}{0.5}
\pgfmathsetmacro{\Ay}{1.4}
\pgfmathsetmacro{\Bx}{0.6}
\pgfmathsetmacro{\By}{0.6}
\pgfmathsetmacro{\Cx}{1}
\pgfmathsetmacro{\Cy}{1.3}
\pgfmathsetmacro{\Dx}{1.1}
\pgfmathsetmacro{\Dy}{2}
\pgfmathsetmacro{\Ex}{1.3}
\pgfmathsetmacro{\Ey}{0.8}


\draw[->] (1.8,\tline) -- (5.1,\tline);
\node at (3.5,\tline-0.3) {$(\TG',v')$};

\foreach \x in {1,...,2}
{
\draw[fill=gray!90!black,opacity=0.2] (\x*\dist,0) -- (\x*\dist+\w,\inh) -- (\x*\dist+\w,\inh+\h) -- (\x*\dist,\h) -- cycle;

\node at (\x*\dist + 0.87 * \w, \inh+ \h - 0.3) {$G_{\x}'$};
\node at (\x*\dist + 0.6 * \w, 0) {$t_{\x}=\x$};
\draw[-] (\x*\dist + 0.6 * \w, \tline+0.05) -- (\x*\dist + 0.6 * \w, \tline-0.05);

\node (A\x) at (\x*\dist+\Ax, 0+\Ay) {};
\node (B\x) at (\x*\dist+\Bx, 0+\By) {};
\node (C\x) at (\x*\dist+\Cx, 0+\Cy) {};
\node (D\x) at (\x*\dist+\Dx, 0+\Dy) {};
\node (E\x) at (\x*\dist+\Ex, 0+\Ey) {};
}


\node[dot=9pt,draw=black, fill=mywhite] at (C1) {$v'$};
\node[dot=9pt,draw=black, fill=myred] at (C2) {$v'$};
\end{tikzpicture}
    \end{minipage}
        \caption{Pointed temporal graphs $(\TG,v)$ and $(\TG',v')$, both including two snapshots, used as an counterexample in the proof of Theorem~\ref{thm:globTGNNweak}. 
        Here, colour $c_1$ is denoted by a red filling (applies for node $v$ in $G_1$, $G_2$ and node $v'$ in $G_2'$).}
    \label{fig:glob_ex}
\end{figure}

\globTGNNweak*
\begin{proof}
    Let $\varphi = \prev{} c_1$. It is evident that $\varphi$ is satisfied by all pointed
    temporal graphs $(\TG,v)$ of length $n \ge 2$ such that $v$ was of colour $c_1$ at
    timestamp $t_1$.

    Consider the two pointed temporal graphs $(\TG,v)$ and $(\TG',v')$, as specified
    by Figure~\ref{fig:glob_ex}. It is clear that $(\TG,v) \models \varphi$ and
    $(\TG',v') \not\models \varphi$. Now, let $T \in \globTGNN[\mathcal{M},\mathcal{Q}, \circ]$ 
    for some $\mathcal{M}$, $\mathcal{Q}$, and $\circ$. The argument is simple:
    We have $h^{0}_v(t_2) = h^{(0)}_{v'}(t_2)$ in the respective computation of $T(\TG,v)$
    and $T(\TG',v')$, and the input to $\agg$ is the empty set in both cases, meaning that
    its output is $\boldsymbol{0}$ in both cases. Thus, as $T$ either accepts 
    both or none of the temporal graphs $(\TG,v)$ and $(\TG',v')$, meaning that it does not accept 
    the exact set of temporal graphs that satisfy $\varphi$.
\end{proof}

\subsection{Proof of Theorem~\ref{sec:variants_logic2tnn;thm:globaltgnn}}
\label{proof:globtgnn}
In the following result, we utilise time2vec functions \cite{Kazemi2019} in the 
constructed TGNN. A formal definition can be found in Appendix~\ref{app:definitions}.
We remark that in the following result we exclusively utilise the $0$-th element of $\mathsf{t2v}$ functions $\phi$, indicating
that the result is independent of the exact form of activation $\sigma$ used in these functions.
\globaltgnn*
\begin{proof}
    Let $\varphi \in \mathcal{L}_2$ and let 
    \begin{displaymath}
           \varphi_1, \dotsc, \varphi_{m}, \varphi_{m+1}, \dotsc, \varphi_n
    \end{displaymath}
    be an enumeration of the subformulas of $\varphi$ such that all atomic
    formulas are the $\varphi_1, \dotsc, \varphi_m$, and $\varphi_i \in \mathit{sub}(\varphi_j)$
    implies $i \le j$. In particular, we consider an enumeration such that $\varphi_n = \varphi$.

    The global (in time) TGNN $\model_\varphi = (M,\phi,\out) \in \globTGNN[\hat{\mathcal{M}}_\msg, \mathcal{Q}_\mathsf{time2vec}, \con]$, 
    where $\phi$ is some $\mathsf{t2v}$ function with $w_0=1$, $b_0=0$, meaning the $0$-th element is the identity, 
    and $\out(x_1, \dots, x_{2n}) = \trelu(x_n)$. We note that $M \in \hat{\mathcal{M}}_\msg$ which means that aggregation 
    is given by $\sum \msg(x)$, where $\msg$ is realised by some one layer FNN $N_\msg$ with truncated-ReLU activation.
    The MPNN $M$ is build such that for each $\varphi_i$ there is a layer $l^{(i)}$ in the same manner as done in 
    Theorem~\ref{sec:logic2tgnn;thm:ptlk}. However, each layer has input dimension $n$ and output dimension $n$, which 
    stands in contrast to previous constructions. This is due to the fact that, the semantics at previous or past timestamps are 
    respected via aggregation. This works as follows.

    Assume that $\varphi_i$ is Boolean. Then, it is handled in layer $l^{(i)}$ exactly as shown in the proof
    of Theorem~\ref{sec:logic2tgnn;thm:ptlk}. In the case of $\varphi_i = Q\psi$ with $Q \in \{\prev{}, \past{}\}$, due to the way $\mathcal{L}_2$ is
    defined, these subformulas must occur in form of $Q\Diamond\chi$.
    For modal subformulas $\Diamond$, we therefore distinguish three cases, namely that $\prev{}\Diamond\varphi_j$, $\past{}\Diamond\varphi_j$,
    or something else. In the case that $\prev{}\Diamond\varphi_j$, we build $\msg$, represented
    by $N_\msg$, such that it maps dimension $j$ identically if $\phi(t)_0 = -1$; otherwise, it maps
    dimension $j$ to $0$. In the case $\past{}\Diamond\varphi_j$, we build $N_\msg$ such that it maps dimension
    $j$ identically if $\phi(t)_0 \leq -1$; otherwise, it maps dimension $j$ to $0$. In the third case, we build
    $N_\msg$ such that it maps dimension $j$ identically if $\phi(t)_0 = 0$; otherwise, it maps dimension $j$
    to $0$. We remark that this involves simple FNN constructions as done in other studies such as \cite{SL22}.
    Informally, the way we construct $N_\msg$ is to filter the correct temporal information.
    Otherwise, $\Diamond\psi$ is handled as seen in Theorem~\ref{sec:logic2tgnn;thm:ptlk}.

    Correctness is again shown via induction over time and subformulae. The key insight here is that,
    while we do not directly evaluate formulas of the form $Q\psi$ with $Q \in \{\prev{}, \past{}\}$, 
    the definition of $\mathcal{L}_2$ ensures that such formulas are of the form
    $Q\Diamond\psi$. Regarding $Q\Diamond\psi$, the interactions between $\phi$ and $\msg$
    ensure that only the information from the previous (in the case of $Q = \prev{}$) or all previous
    (in the case of $Q = \past{}$) timepoints is considered in the aggregation. Other than this, the
    arguments are the same as in Theorem~\ref{sec:logic2tgnn;thm:ptlk}.
\end{proof}




\newpage
\section*{NeurIPS Paper Checklist}

The checklist is designed to encourage best practices for responsible machine learning research, addressing issues of reproducibility, transparency, research ethics, and societal impact. Do not remove the checklist: {\bf The papers not including the checklist will be desk rejected.} The checklist should follow the references and follow the (optional) supplemental material.  The checklist does NOT count towards the page
limit. 

Please read the checklist guidelines carefully for information on how to answer these questions. For each question in the checklist:
\begin{itemize}
    \item You should answer \answerYes{}, \answerNo{}, or \answerNA{}.
    \item \answerNA{} means either that the question is Not Applicable for that particular paper or the relevant information is Not Available.
    \item Please provide a short (1–2 sentence) justification right after your answer (even for NA). 
\end{itemize}

{\bf The checklist answers are an integral part of your paper submission.} They are visible to the reviewers, area chairs, senior area chairs, and ethics reviewers. You will be asked to also include it (after eventual revisions) with the final version of your paper, and its final version will be published with the paper.

The reviewers of your paper will be asked to use the checklist as one of the factors in their evaluation. While "\answerYes{}" is generally preferable to "\answerNo{}", it is perfectly acceptable to answer "\answerNo{}" provided a proper justification is given (e.g., "error bars are not reported because it would be too computationally expensive" or "we were unable to find the license for the dataset we used"). In general, answering "\answerNo{}" or "\answerNA{}" is not grounds for rejection. While the questions are phrased in a binary way, we acknowledge that the true answer is often more nuanced, so please just use your best judgment and write a justification to elaborate. All supporting evidence can appear either in the main paper or the supplemental material, provided in appendix. If you answer \answerYes{} to a question, in the justification please point to the section(s) where related material for the question can be found.

IMPORTANT, please:
\begin{itemize}
    \item {\bf Delete this instruction block, but keep the section heading ``NeurIPS paper checklist"},
    \item  {\bf Keep the checklist subsection headings, questions/answers and guidelines below.}
    \item {\bf Do not modify the questions and only use the provided macros for your answers}.
\end{itemize}


\begin{enumerate}

    \item {\bf Claims}
        \item[] Question: Do the main claims made in the abstract and introduction accurately reflect the paper's contributions and scope?
        \item[] Answer: \answerYes{} 
        \item[] Justification: The claims made in the abstract and introduction are covered by our core results. The detailed 
        results are outlined in the ``Our contribution.'' subsection of the 
        introduction.
    
    \item {\bf Limitations}
        \item[] Question: Does the paper discuss the limitations of the work performed by the authors?
        \item[] Answer: \answerYes{} 
        \item[] Justification: See Section~\ref{sec:outlook} as well as remarks in the pre- or posttext of specific results and definitions. 
        This especially includes the fact that all models and logics are defined rigorously.
    
    \item {\bf Theory assumptions and proofs}
        \item[] Question: For each theoretical result, does the paper provide the full set of assumptions and a complete (and correct) proof?
        \item[] Answer: \answerYes{} 
        \item[] Justification: For each result, we included a proof sketch in the main body and a full proof in the appendix. Furthermore, we defined each model and framework rigorously.
    
        \item {\bf Experimental result reproducibility}
        \item[] Question: Does the paper fully disclose all the information needed to reproduce the main experimental results of the paper to the extent that it affects the main claims and/or conclusions of the paper (regardless of whether the code and data are provided or not)?
        \item[] Answer: \answerNA{} 
        \item[] Justification:

    \item {\bf Open access to data and code}
        \item[] Question: Does the paper provide open access to the data and code, with sufficient instructions to faithfully reproduce the main experimental results, as described in supplemental material?
        \item[] Answer: \answerNA{} 
        \item[] Justification:

    \item {\bf Experimental setting/details}
        \item[] Question: Does the paper specify all the training and test details (e.g., data splits, hyperparameters, how they were chosen, type of optimizer, etc.) necessary to understand the results?
        \item[] Answer: \answerNA{} 
        \item[] Justification: 
    
    \item {\bf Experiment statistical significance}
        \item[] Question: Does the paper report error bars suitably and correctly defined or other appropriate information about the statistical significance of the experiments?
        \item[] Answer: \answerNA{} 
        \item[] Justification:
    
    \item {\bf Experiments compute resources}
        \item[] Question: For each experiment, does the paper provide sufficient information on the computer resources (type of compute workers, memory, time of execution) needed to reproduce the experiments?
        \item[] Answer: \answerNA{} 
        \item[] Justification: 
        
    \item {\bf Code of ethics}
        \item[] Question: Does the research conducted in the paper conform, in every respect, with the NeurIPS Code of Ethics \url{https://neurips.cc/public/EthicsGuidelines}?
        \item[] Answer: \answerYes{} 
        \item[] Justification: 

    \item {\bf Broader impacts}
        \item[] Question: Does the paper discuss both potential positive societal impacts and negative societal impacts of the work performed?
        \item[] Answer: \answerNA{} 
        \item[] Justification:
        
    \item {\bf Safeguards}
        \item[] Question: Does the paper describe safeguards that have been put in place for responsible release of data or models that have a high risk for misuse (e.g., pretrained language models, image generators, or scraped datasets)?
        \item[] Answer: \answerNA{} 
        \item[] Justification:
    
    \item {\bf Licenses for existing assets}
        \item[] Question: Are the creators or original owners of assets (e.g., code, data, models), used in the paper, properly credited and are the license and terms of use explicitly mentioned and properly respected?
        \item[] Answer: \answerNA{} 
        \item[] Justification:
    
    \item {\bf New assets}
        \item[] Question: Are new assets introduced in the paper well documented and is the documentation provided alongside the assets?
        \item[] Answer: \answerNA{} 
        \item[] Justification: 
    
    \item {\bf Crowdsourcing and research with human subjects}
        \item[] Question: For crowdsourcing experiments and research with human subjects, does the paper include the full text of instructions given to participants and screenshots, if applicable, as well as details about compensation (if any)? 
        \item[] Answer: \answerNA{} 
        \item[] Justification:
    
    \item {\bf Institutional review board (IRB) approvals or equivalent for research with human subjects}
        \item[] Question: Does the paper describe potential risks incurred by study participants, whether such risks were disclosed to the subjects, and whether Institutional Review Board (IRB) approvals (or an equivalent approval/review based on the requirements of your country or institution) were obtained?
        \item[] Answer: \answerNA{} 
        \item[] Justification: 
    
    \item {\bf Declaration of LLM usage}
        \item[] Question: Does the paper describe the usage of LLMs if it is an important, original, or non-standard component of the core methods in this research? Note that if the LLM is used only for writing, editing, or formatting purposes and does not impact the core methodology, scientific rigorousness, or originality of the research, declaration is not required.
        \item[] Answer: \answerNA{} 
        \item[] Justification:
    
    \end{enumerate}

\end{document}